\documentclass{article}

\usepackage{microtype}
\usepackage{graphicx}
\usepackage{bm}
\usepackage{subcaption}
\usepackage{booktabs} 

\usepackage{hyperref}

\usepackage[preprint]{icml2026}

\usepackage{amsmath}
\usepackage{amssymb}
\usepackage{mathtools}
\usepackage{amsthm}
\usepackage{array}
\usepackage{tabularx}
\usepackage[capitalize,noabbrev]{cleveref}
\usepackage{amsfonts}
\usepackage{multirow} 
\usepackage{colortbl}
\usepackage{wrapfig}

\usepackage[most]{tcolorbox} 
\newtcolorbox[list inside=prompt,auto counter]{prompt}[1][]{
    colbacktitle=black!60,
    coltitle=white,
    fontupper=\footnotesize,
    boxsep=5pt,
    left=0pt,
    right=0pt,
    top=0pt,
    bottom=0pt,
    boxrule=1pt,
    #1,
}

\theoremstyle{plain}
\newtheorem{theorem}{Theorem}[section]

\newtheorem{corollary}[theorem]{Corollary}
\theoremstyle{definition}

\newtheorem{assumption}[theorem]{Assumption}
\theoremstyle{remark}

\newcolumntype{Y}{>{\centering\arraybackslash}X}

\usepackage[textsize=tiny]{todonotes}

\NewDocumentCommand{\yafu}
{ mO{} }{\textcolor{blue}{\textsuperscript{\textit{yafu}}\textsf{\textbf{\small[#1]}}}}

\icmltitlerunning{FaithRL: Learning to Reason Faithfully through Step-Level Faithfulness Maximization}

\begin{document}

\twocolumn[

\icmltitle{FaithRL: Learning to Reason Faithfully through Step-Level Faithfulness Maximization}




  \begin{icmlauthorlist}
    \icmlauthor{Runquan Gui}{ustc,shlab}
    \icmlauthor{Yafu Li}{shlab}
    \icmlauthor{Xiaoye Qu}{shlab}
    \icmlauthor{Ziyan Liu}{ustc}
    \icmlauthor{Yeqiu Cheng}{ustc}
    \icmlauthor{Yu Cheng}{cuhk}
\end{icmlauthorlist}

\icmlaffiliation{shlab}{Shanghai AI Laboratory}
\icmlaffiliation{ustc}{University of Science and Technology of China}
\icmlaffiliation{cuhk}{The Chinese University of Hong Kong}

\icmlcorrespondingauthor{Yafu Li}{yafuly@gmail.com}
  \icmlkeywords{Machine Learning, ICML}

  \vskip 0.3in
]




\printAffiliationsAndNotice{Work was done during Runquan Gui's internship at Shanghai AI Laboratory. Yafu Li is the Project Lead.}

\begin{abstract}
Reinforcement Learning with Verifiable Rewards (RLVR) has markedly improved the performance of Large Language Models (LLMs) on tasks requiring multi-step reasoning.
However, most RLVR pipelines rely on sparse outcome-based rewards, providing little supervision over intermediate steps and thus encouraging over-confidence and spurious reasoning, which in turn increases hallucinations.
To address this, we propose \textbf{FaithRL}, a general reinforcement learning framework that directly optimizes reasoning faithfulness.
We formalize a faithfulness-maximization objective and theoretically show that optimizing it mitigates over-confidence.
To instantiate this objective, we introduce a geometric reward design and a faithfulness-aware advantage modulation mechanism that assigns step-level credit by penalizing unsupported steps while preserving valid partial derivations.
Across diverse backbones and benchmarks, FaithRL consistently reduces hallucination rates while maintaining (and often improving) answer correctness.
Further analysis confirms that FaithRL increases step-wise reasoning faithfulness and generalizes robustly.
Our code is available at \url{https://github.com/aintdoin/FaithRL}.
\end{abstract}

\section{Introduction}
Reinforcement Learning with Verifiable Rewards (RLVR) has substantially advanced LLM reasoning capabilities in complex domains such as mathematics and coding~\cite{guo2025deepseek}.
However, current paradigms relying on sparse outcome-based rewards often neglect intermediate process validity.
Since successfully guessing the answer yields a positive reward while explicit refusal is often penalized, models are incentivized toward \textit{over-confidence}, favoring speculative answering rather than acknowledging uncertainty~\cite{chen2025dont, openai2025hallucination}.
Furthermore, the scalar nature of the reward obscures step-level fidelity, reinforcing erroneous steps that coincidentally yield correct answers while penalizing valid partial derivations~\cite{agarwal2024faithfulness, huang2025survey}.
Consequently, restoring faithfulness necessitates that models act honestly by refusing to predict when evidentiary support is insufficient, and logically by ensuring every reasoning step contributes validly to the final result~\cite{wang2025comprehensive,chen2025learning}.

\begin{figure}[t]
\begin{center}
\centerline{\includegraphics[width=0.48\textwidth]{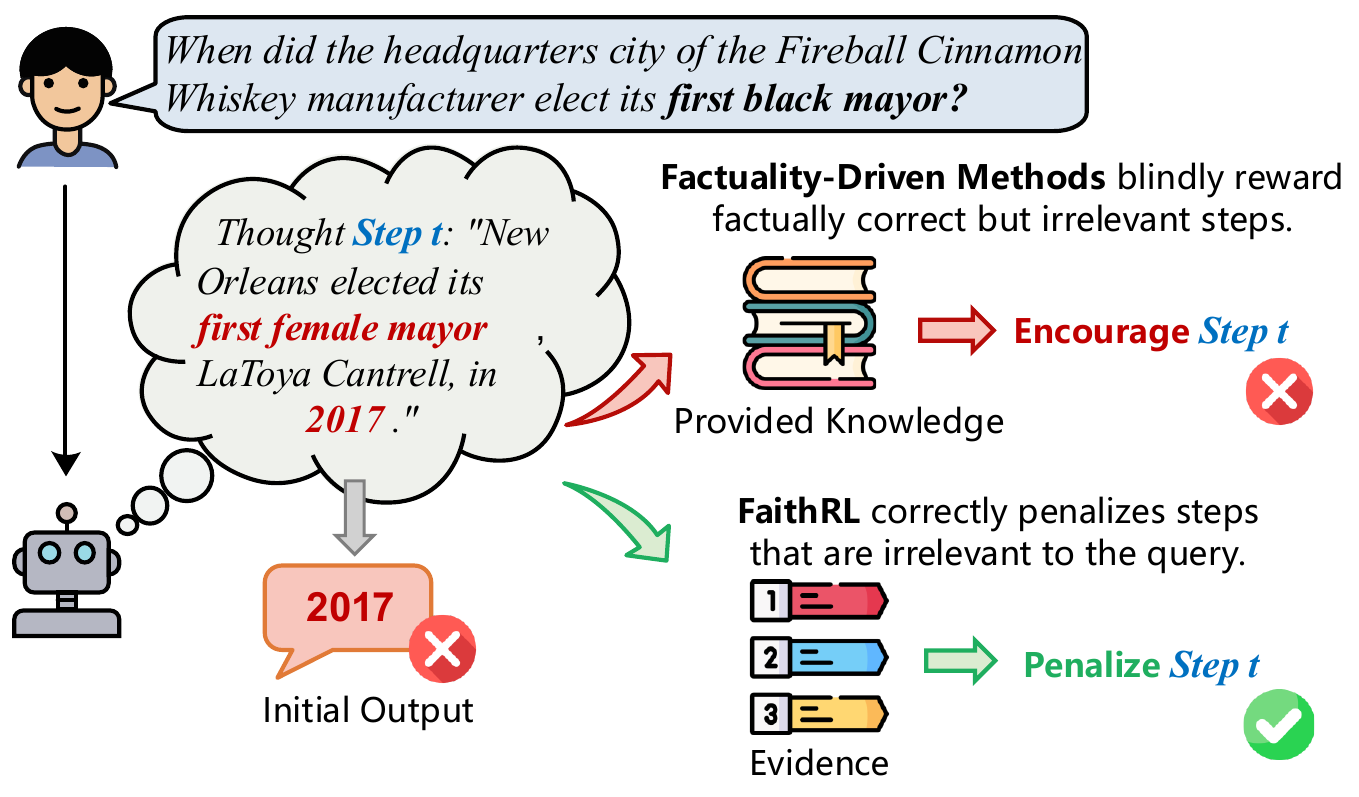}}
\caption{Comparison between factuality-driven methods and FaithRL. Unlike factuality-driven methods that induce hallucinations by rewarding correct but unfaithful steps, FaithRL penalizes such irrelevance to ensure a faithful and correct answer.}
\label{Figure1}
\end{center}
\vspace{-30pt}
\end{figure}

Recently, numerous approaches have been proposed to improve model faithfulness.
A prominent line of work constructs refusal datasets by treating consistent model failures as unanswerable instances and supervising models to generate explicit refusals~\cite{zhang2024r}.
However, tying refusal behavior to early-stage failures ignores the fact that model capabilities evolve during training.
As a result, such methods can bias models toward \textit{over-conservativeness}, degrading reasoning performance while inflating refusal rates.
Complementary efforts focus on verifying intermediate reasoning steps, typically using factuality-driven metrics~\cite{ren2025knowrl, lireasoning}.
While effective at enforcing factual correctness, these approaches often struggle to guarantee logical soundness: as illustrated in Figure~\ref{Figure1}, models may produce chains composed of factually correct but irrelevant statements, failing to form a valid deductive path.

To address these challenges, we propose \underline{Faith}fulness-aware \underline{R}einforcement \underline{L}earning (\textbf{FaithRL}), a general reinforcement learning (RL) framework that promotes faithful reasoning. 
Instead of rewarding only the final answer, it encourages trajectories whose intermediate steps are supported by the required evidence and discourages unsupported steps even when they accidentally lead to a correct outcome.
Concretely, we (i) formalize \textit{maximizing reasoning faithfulness} as an objective and show that it avoids degenerate behaviors such as over-confidence and over-conservativeness; (ii) design a \textit{geometric reward} that balances answering and refusal according to baseline capability, preventing both reward-hacking via guessing and blanket refusal; and (iii) introduce \textit{faithfulness-aware advantage modulation}, which assigns step-level credit by filtering positive updates to faithful steps and concentrating penalties on unfaithful ones.

FaithRL yields consistent gains across benchmarks.
Compared with prior RLVR methods, it reduces hallucination rates by an average of \textbf{4.7} points while improving correctness by \textbf{1.6} points on the \textit{full} variants of 2WikiMultiHopQA~\cite{ho2020constructing}, HotpotQA~\cite{yang2018hotpotqa}, and MuSiQue~\cite{trivedi2022musique}.
It also generalizes strongly: on out-of-distribution tasks such as MATH500~\cite{hendrycks2021measuring} and GSM8k~\cite{cobbe2021training}, FaithRL improves accuracy by over 10.8 points on average and reduces hallucination rates by 8.4 points.
Finally, FaithRL steadily increases the faithful reasoning ratio during training, improving the proportion of faithful steps by 31.1\% in-domain and 6.5\% out-of-distribution.

\begin{itemize}
    \item We propose the novel optimization objective of \textit{maximizing reasoning faithfulness}, theoretically proving that it prevents the model from lapsing into over-confidence or over-conservativeness.
    \item We introduce FaithRL, a unified RL framework designed to strictly align model incentives with this objective. To the best of our knowledge, we are the first work to propose mitigating hallucinations by directly optimizing reasoning faithfulness via objective evidence verification, laying a foundation for future research.
    \item FaithRL consistently outperforms strong baselines across three tasks, achieving an average \textbf{4.7} points reduction in hallucination and a \textbf{1.6} points gain in correctness over RLVR methods, yielding a \textbf{16.0} points improvement in the truthful helpfulness score defined in Section~\ref{subsec:faithrl}.
\end{itemize}

\section{Related Work}\label{sec:related_work}
\paragraph{Hallucination Mitigation}
Mitigating hallucinations necessitates a delicate balance between maximizing response correctness and ensuring effective refusal of unanswerable queries.
Refusal-Aware Instruction Tuning~(RAIT) explicitly trains models to generate specific refusal indicators for unknown questions, such as outputting "I don't know"~\cite{cheng2024can, yang2024alignment} or appending uncertainty markers~\cite{zhang2024r}.
Although advanced methods like CRaFT~\cite{zhu2025utilize} attempt to alleviate over-refusal by concurrently fine-tuning on high-certainty and unsolvable samples, these supervised approaches may inadvertently constrain general reasoning capabilities.
Confidence-based strategies leverage model uncertainty to determine abstention decisions at inference time~\cite{jurayj2025your}.
Implementing these strategies effectively necessitates precise calibration techniques~\cite{sun2025detection, wang2025joint, damani2025beyond}, and the reliance on fixed thresholds can be sensitive to specific task distributions, potentially affecting generalization.
Recently, RLVR methods have explored incentivizing refusal by penalizing hallucinations through multi-dimensional~\cite{chen2025learning} or multi-level rewards~\cite{wei2025truthrl}; nevertheless, such coarse-grained outcome rewards might provide limited guidance for optimizing the fine-grained intermediate reasoning process.

\vspace{-5pt}

\paragraph{Reasoning Process Verification}
In multi-hop QA, LLMs optimized via outcome-driven RL often generate reasoning chains that are irrelevant to the target objective.
Existing literature frequently attributes this faithfulness issue primarily to factual errors~\cite{huang2025improving, agarwal2024faithfulness, chen2025reasoning}.
To address this, prevalent verification methods decompose responses into independent atomic claims for retrieval-based fact-checking, as exemplified by HalluMeasure~\cite{akbar2024hallumeasure}, KnowHalu~\cite{zhang2024knowhalu}, and others~\cite{paul2024making, ren2025knowrl, lireasoning}.
While effective at mitigating hallucinations arising from factual inaccuracies, these claim-centric strategies may be less effective in addressing invalid logical processes.

\section{Preliminary}

\subsection{Problem Formulation}
\label{subsec:problem_formulation}
We define the multi-hop QA task over a universal knowledge base $\mathcal{S}$, a set of questions $\mathcal{Q}$, and an answer space $\mathcal{A}$. For each query $q \in \mathcal{Q}$, we designate a specific context $\mathcal{K}(q) \subset \mathcal{S}$, which consists of a set of documents associated with the query. An LLM $\mathcal{M}$ takes the pair $(q, \mathcal{K}(q))$ as input to generate a sequence of tokens $\tau = (t_1, \dots, t_L)$. To facilitate step-level analysis, we structurally decompose the trajectory into discrete reasoning steps and a final answer, denoted as $\tau = (s_1, \dots, s_k, a)$, where $s_{1:k}$ represents the reasoning process. Let $A^* \in \mathcal{A}$ be the ground truth.

We postulate that $\mathcal{S}$ is informationally complete:

\begin{assumption}[\textbf{Knowledge Completeness}]
\label{asm:completeness}
For any $q \in \mathcal{Q}$, $\mathcal{S}$ contains sufficient information to derive the correct answer $A^*$.
\end{assumption}

Based on this assumption, we define the evidence set $\mathcal{E}(q)$ as the minimal subset of $\mathcal{S}$ required to answer $q$. In standard multi-hop reasoning tasks, the reasoning chain is universally designed to be unique, thereby ensuring the uniqueness of $\mathcal{E}(q)$. From this definition, we derive the following corollary regarding the answerability of a question:

\begin{corollary}[\textbf{Answerability Partition}]
Based on the evidence availability, the query set $\mathcal{Q}$ is partitioned into two disjoint subsets: the Answerable Set $\mathcal{Q}_{ans}$ and the Unanswerable Set $\mathcal{Q}_{unans}$.
\begin{align}
    \mathcal{Q}_{ans} &= \{ q \in \mathcal{Q} \mid \mathcal{E}(q) \subseteq \mathcal{K}(q) \} \\
    \mathcal{Q}_{unans} &= \mathcal{Q} \setminus \mathcal{Q}_{ans}
\end{align}
\end{corollary}

For queries identified as unanswerable ($q \in \mathcal{Q}_{unans}$), we explicitly modify the ground truth to $A^* = \text{``IDK''}$ and the evidence set to $\mathcal{E}(q) = \mathcal{E}(q) \cup \{\text{IDK}\}$.

\subsection{Process-Outcome Consistency}

Let $\mathcal{T}$ denote the sample space of all possible reasoning trajectories. We define two operators for any $\tau \in \mathcal{T}$:
\begin{enumerate}
    \item \textbf{Knowledge extraction} $\phi: \mathcal{T} \rightarrow 2^{\mathcal{S}}$ (the power set of $\mathcal{S}$) identifies the set of knowledge items used in $\tau$.
    \item \textbf{Answer Parsing} $\psi: \mathcal{T} \rightarrow \mathcal{A} \cup \{\text{IDK}\}$ extracts the final derived answer.
\end{enumerate}

We further define the set of \textit{hallucinated answers} as $\mathcal{H} = \mathcal{A} \setminus \{A^*, \text{IDK}\}$. Modeling the LLM as a stochastic policy $\pi_\theta$, we propose dual assumptions bridging evidence and answers. 
First, we consider correctness sufficiency. We posit that covering the evidence set guarantees the derivation of the correct answer, regardless of redundant steps:

\begin{assumption}[\textbf{Correctness Sufficiency}]
\label{asm:iff_evidence}
\begin{equation}
    \mathcal{E}(q) \subseteq \phi(\tau) \implies \psi(\tau) = A^*
\end{equation}
\end{assumption}

Conversely, we consider hallucination prevention. We assume that restricting the reasoning trajectory strictly to verified evidence precludes the generation of hallucinations:

\begin{assumption}[\textbf{Hallucination Prevention}]
\label{asm:faithfulness}
\begin{equation}
    \phi(\tau) \subseteq \mathcal{E}(q) \implies \psi(\tau) \notin \mathcal{H}
\end{equation}
\end{assumption}

We acknowledge that Assumptions \ref{asm:iff_evidence} and \ref{asm:faithfulness} represent theoretical idealizations of LLM behavior required to construct a tractable optimization framework. While logical entailment does not strictly guarantee probabilistic generation in all stochastic scenarios, empirical verification on standard benchmarks (detailed in Appendix~\ref{app:assumption_validation}) demonstrates that these assumptions hold with high probability (validity $> 80\%$), confirming their high applicability in practice.

\begin{figure*}[t]
    \centering 
    \includegraphics[width=\textwidth, trim=0 0 0 0, clip]{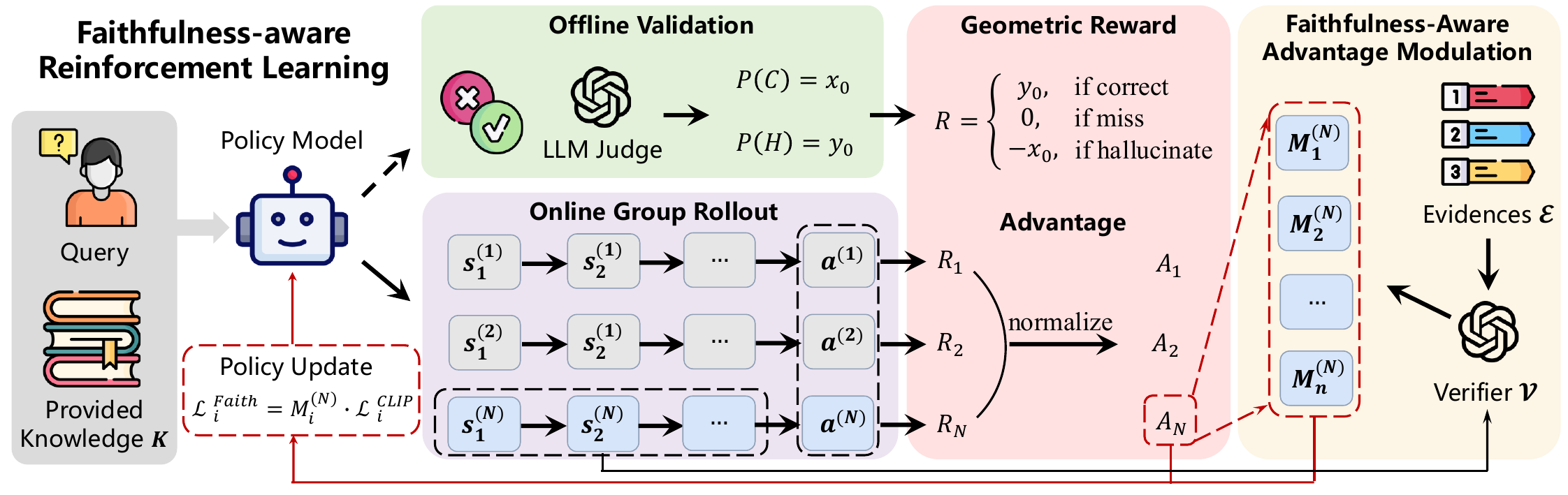}
    \caption{Overall framework of FaithRL. Taking a query and knowledge $\mathcal{K}$, the model first uses offline validation to set baseline rates ($x_0, y_0$). During online group rollout, these rates form a geometric reward based on answer correctness. Simultaneously, the verifier $\mathcal{V}$ checks alignment with evidence $\mathcal{E}$ to derive token-level advantage modulation for the policy update.}
    \label{Figure3}
    \vspace{-10pt}
\end{figure*}

\subsection{Group Relative Policy Optimization (GRPO)}

GRPO~\cite{shao2024deepseekmath} serves as the foundation for our RL framework, characterized by its elimination of the value function in favor of group-based advantage estimation.
For a given query $q$, GRPO samples a group of $N$ outputs $\{\tau_1, \dots, \tau_N\}$ from the old policy $\pi_{\theta_{old}}$.
A reward function $r(\tau_i)$ evaluates the final outcome, typically providing a sparse binary signal:
\begin{equation}
    \label{eq:grpo_reward}
    r(\tau_i) = 
    \begin{cases} 
        1, & \text{if } \psi(\tau_i) = A^* \\ 
        0, & \text{otherwise} 
    \end{cases}
\end{equation}

The advantage $A_i$ for the $i$-th output is computed by normalizing rewards within the group: $A_i = (r(\tau_i) - \mu_r) / \sigma_r$, where $\mu_r$ and $\sigma_r$ denote the mean and standard deviation of the group rewards, respectively.

GRPO adapts the PPO surrogate objective to this group setting. Formally, the optimization objective is defined as the expectation over the token-level clipped surrogate loss:
\begin{equation}
    \label{eq:grpo_objective}
    \mathcal{J}_{GRPO}(\theta) = \mathbb{E}_{q, \{\tau_i\}} \left[ \frac{1}{N} \sum_{i=1}^N \frac{1}{|\tau_i|} \sum_{t=1}^{|\tau_i|} \mathcal{L}^{CLIP}_{i,t}(\theta) \right]
\end{equation}

Here, $\mathcal{L}^{CLIP}_{i,t}(\theta)$ denotes the standard clipped objective term: $\mathcal{L}^{CLIP}_{i,t}(\theta) = \min \left( \rho_{i,t} A_i, \text{clip}(\rho_{i,t}, 1-\epsilon, 1+\epsilon) A_i \right)$,
where $\rho_{i,t} = \dfrac{\pi_\theta(\tau_{i,t} | q, \tau_{i,<t})}{\pi_{\theta_{old}}(\tau_{i,t} | q, \tau_{i,<t})}$ represents the probability ratio between the current and reference policies. The KL divergence regularization term is omitted here for brevity.

\section{Methodology}
In this section, we introduce the FaithRL framework in detail, with the overall process illustrated in Figure~\ref{Figure3}.

\subsection{Optimization Objectives}

We assess model performance by partitioning the trajectory space $\mathcal{T}$ into three distinct outcomes.
The \textit{correct set} $S_{c} = \{ \tau \in \mathcal{T} \mid \psi(\tau) = A^* \}$ includes trajectories where the model provides the correct answer or properly refuses an unanswerable query.
The \textit{miss set} $S_{m} = \{ \tau \in \mathcal{T} \mid \psi(\tau) = \text{IDK} \land A^* \neq \text{IDK} \}$ represents false refusals, occurring only when the model incorrectly declines to answer a valid question.
All other errors fall into the \textit{hallucination Set} $S_{h} = \{ \tau \in \mathcal{T} \mid \psi(\tau) \in \mathcal{H} \}$.
We denote the probabilities of these outcomes for a policy $\pi_\theta$ as correctness $P(C)$, miss rate $P(M)$, and hallucination rate $P(H)$.

Conventional optimization paradigms in QA typically prioritize one of the following objectives:
\begin{itemize}
    \item \textbf{Objective A (maximize correctness):} Maximize $P(C)$. This is the standard objective for general QA~\cite{li2024teaching, huang2025improving}.
    \item \textbf{Objective B (minimize hallucination):} Minimize $P(H)$. This is often prioritized in safety-critical applications~\cite{liu2024hallucination, luo2024zero}.
\end{itemize}

In contrast, we propose \textbf{objective C (maximize reasoning faithfulness)}. This objective targets the \textit{reasoning-faithful set} $S_{f} = \{ \tau \in \mathcal{T} \mid \phi(\tau) = \mathcal{E}(q) \}$, which consists of trajectories where the used knowledge strictly matches the required evidence. Our goal is to maximize $P(F) = P(\tau \in S_{f})$. Based on Assumptions \ref{asm:iff_evidence} and \ref{asm:faithfulness}, this objective naturally unifies accuracy and honesty because strict adherence to evidence guarantees the correct answer while precluding hallucination. This formulation ensures that every step in the reasoning chain makes a necessary contribution to the deduction, rather than merely supporting a correct final guess.

The following theorem demonstrates that while standard objectives push models toward extreme behaviors, optimizing for reasoning faithfulness aligns refusal decisions with model capability.

\begin{theorem}[\textbf{Asymptotic stability of refusal strategy}]
Consider a model with imperfect intrinsic capability. Optimizing the three defined objectives leads to distinct asymptotic behaviors for the miss rate $P(M)$:
\begin{enumerate}
    \item \textbf{Objective A (maximize correctness)} forces $P(M) \to 0$. This results in \textbf{over-confidence}, where the model speculates on all queries to maximize rewards.
    \item \textbf{Objective B (minimize hallucination)} drives $P(M) \to 1$. This leads to \textbf{over-conservativeness}, causing the model to refuse answerable queries to avoid hallucination risks.
    \item \textbf{Objective C (maximize reasoning faithfulness)} ensures \textbf{stability}. It maintains the miss rate at a stable equilibrium determined by the model's intrinsic capability, effectively preventing the degenerate collapse observed in objectives A and B.
\end{enumerate}
\end{theorem}

Detailed proofs are provided in Appendix \ref{app:proofs}.

\subsection{Faithfulness-aware Reinforcement Learning}
\label{subsec:faithrl}
We propose FaithRL, a RL framework designed to maximize the faithfulness objective $P(F)$. To achieve this, FaithRL integrates a geometric reward mechanism that aligns outcome-level optimization with the model's intrinsic capability, and a faithfulness-aware advantage modulation strategy that enforces step-level process supervision.

\subsubsection{Geometric Reward Design}

Existing strategies for optimizing reasoning tasks often rely on independent metrics for correctness and hallucination, which may fail to accurately quantify the genuine performance gain after optimization. To simultaneously address helpfulness and honesty without relying on arbitrary hyperparameters, we adopt the \textbf{truthful helpfulness score (THS)}~\cite{zhu2025utilize} as a higher-order metric that provides a geometrically intuitive fusion of utility and honesty.

As seen in Figure~\ref{Figure4}, we represent the model's capability as a point on a two-dimensional plane, where the x-axis denotes the correctness rate and the y-axis represents the hallucination rate. Let $E_0=(x_0, y_0)$ be the initial capability point corresponding to the baseline, $E_1$ be the point after optimization, and $O$ be the origin. The ideal strategy $E^*(1,0)$ denotes 100\% correctness with zero hallucinations. THS is defined as the \textbf{normalized signed area}:
\begin{equation}
    \text{THS} = \frac{\vec{S}_{\triangle OE_{0}E_{1}}}{\vec{S}_{\triangle OE_{0}E^*}}
\end{equation}
This metric effectively captures the degree of model refinement. Crucially, maximizing this signed area aligns perfectly with the dual goals of increasing correctness and reducing hallucinations, providing a robust measure of optimization quality. Detailed mathematical definitions are provided in Appendix~\ref{app:ths_details}.


\begin{figure}[t]
\begin{center}
\centerline{\includegraphics[width=0.35\textwidth]{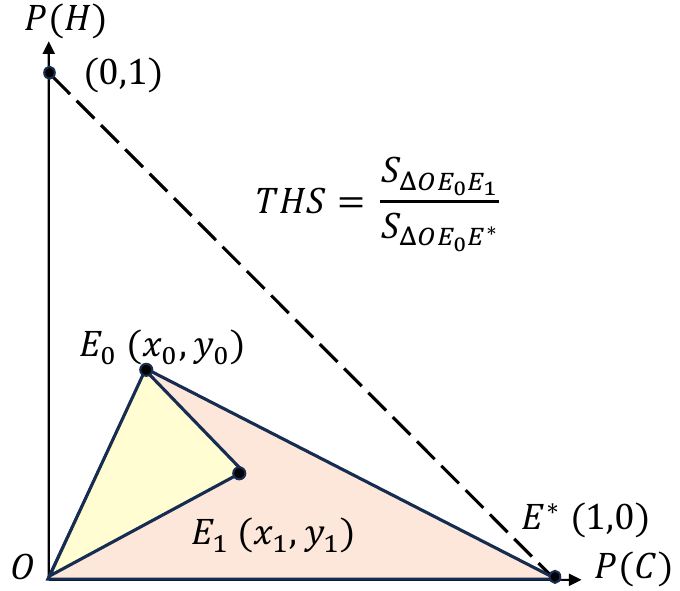}}
\caption{Truthful Helpfulness Score (THS).}
\vspace{-20pt}
\label{Figure4}
\end{center}
\end{figure}

Guided by THS, we derive a novel reward mechanism where the coefficients are not arbitrary hyperparameters but are derived directly from the baseline coordinates $(x_0, y_0)$:
\begin{equation}
\label{eq:geometric_reward}
    \mathcal{R}_{geo}(\tau) = 
    \begin{cases} 
        +y_0, & \text{if } \tau \in S_{c} \\
        0, & \text{if } \tau \in S_{m} \\
        -x_0, & \text{if } \tau \in S_{h} 
    \end{cases}
\end{equation}
This formulation naturally implements an adaptive regulation strategy. For a model with a high hallucination rate, the reward structure assigns a substantial weight to correctness, which prioritizes improving accuracy. Conversely, for a model that is already proficient with high accuracy, the dominant penalty term compels the model to focus on precision, strictly suppressing the risk of hallucination.

Crucially, this design is not merely heuristic but has a geometric interpretation. We prove that optimizing this capability-dependent reward guarantees that the model strictly follows the optimization trajectory of THS.

\begin{theorem}[\textbf{Geometric alignment with THS}]
\label{thm:ths_gradient}
Let $J(\theta) = \mathbb{E}_{\tau \sim \pi_\theta}[\mathcal{R}_{geo}(\tau)]$ be the expected return under the proposed geometric reward. The optimization gradient is strictly proportional to the gradient of THS with a positive scalar coefficient $C$:
\begin{equation}
    \nabla_\theta J(\theta) = C \cdot \nabla_\theta \text{THS}(\pi_\theta), \quad \text{where } C > 0
\end{equation}
This implies that maximizing the geometric reward is mathematically equivalent to maximizing the signed area improvement relative to the baseline capability.
\end{theorem}
The complete proof is provided in Appendix \ref{app:proof_thm2}.

\subsubsection{Faithfulness-Aware Advantage Modulation}

Geometric rewards optimize the final outcome, but they treat the reasoning process as a black box. To ensure faithfulness, we must verify that every single step contributes validly to the result. This requires supervision at the process level. We introduce \textbf{faithfulness-aware advantage modulation (FAAM)} to strictly enforce this step-wise fidelity.

Adopting the structure from Section~\ref{subsec:problem_formulation}, we decompose the trajectory $\tau$ into discrete reasoning steps $(s_1, \dots, s_k)$ and a final answer. We then employ a verifier $\mathcal{V}$ to assess whether each step $s_j$ is strictly supported by the evidence $\mathcal{E}(q)$. Formally, we define $\mathcal{V}(s_j) = \mathbb{I}[\phi(s_j) \subseteq \mathcal{E}(q)]$, where $1$ signifies a faithful step and $0$ denotes an unfaithful one. This verification is merely an objective check for the utilization of evidence from $\mathcal{E}(q)$ in the step.

We integrate this verification mechanism into the GRPO framework. First, we calculate the advantage $A_i$ for each trajectory using the geometric rewards $\mathcal{R}_{geo}(\tau_i)$ from Equation~\ref{eq:geometric_reward}. To apply supervision at the step level, we introduce a modulation scalar $M_{i,t}$ for each token. Let $\alpha \in [0, 1)$ be the faithfulness coefficient. For any token $t$ within step $s_j$, we define the modulation term as:
\begin{equation}
    M_{i,t} = 
    \begin{cases} 
        (1-\alpha)\mathcal{V}(s_j) + \alpha, & \text{if } A_i > 0 \\
        (1-\alpha)(1 - \mathcal{V}(s_j)) + \alpha, & \text{if } A_i \leq 0
    \end{cases}
\end{equation}

This design functions as a credit assignment filter. For trajectories with a positive advantage ($A_i > 0$), we reinforce valid reasoning: faithful steps get the full reward ($M_{i,t}=1$), and unfaithful steps a discounted reward ($M_{i,t}=\alpha$). For trajectories with a negative advantage, by contrast, unfaithful steps face the full penalty, while faithful steps incur a reduced one. This prevents the model from being unfairly penalized for correct partial reasoning steps, even if the final outcome is wrong.

Building upon the GRPO surrogate objective (Eq.~\ref{eq:grpo_objective}), we incorporate this modulation to define the FaithRL objective:
\begin{equation*}
    \mathcal{J}_{\text{FaithRL}}(\theta) = \mathbb{E}_{q, \{\tau_i\}} \left[ \frac{1}{N} \sum_{i=1}^N \frac{1}{|\tau_i|} \sum_{t=1}^{|\tau_i|} M_{i,t} \cdot \mathcal{L}^{CLIP}_{i,t}(\theta) \right]
\end{equation*}
This objective effectively adjusts the importance of each step during training. By dynamically re-weighting the learning signal, FaithRL ensures that the model is reinforced primarily for reasoning that is both correct and faithful.

\begin{table*}[!t]
\centering
\caption{Main results comparing \textbf{FaithRL} against strong baselines across three multi-hop QA benchmarks. We report \textbf{correct rate (C)} and \textbf{truthful helpfulness score (THS)} where higher is better ($\uparrow$), and \textbf{hallucination rate (H)} where lower is better ($\downarrow$). The best performance is highlighted in \textbf{bold}, and the second best is \underline{underlined}.}
\label{tab:main_results}
\vspace{-0.5em}

\scriptsize 
\setlength{\tabcolsep}{1pt}
\begin{tabularx}{\textwidth}{
    l               
    @{\hskip 4pt}   
    YYY             
    @{\hskip 4pt}   
    YYY             
    @{\hskip 4pt}   
    YYY             
    @{\hskip 4pt}   
    YYY             
}
\toprule
& \multicolumn{3}{c}{\textbf{2WikiMultihopQA-Full}} 
& \multicolumn{3}{c}{\textbf{HotpotQA-Full}} 
& \multicolumn{3}{c}{\textbf{MuSiQue-Full}} 
& \multicolumn{3}{c}{\textbf{Average}} \\ 

\cmidrule(lr){2-4} \cmidrule(lr){5-7} \cmidrule(lr){8-10} \cmidrule(lr){11-13}

\textbf{Method} 
& \textbf{C} ($\uparrow$) & \textbf{H} ($\downarrow$)& \textbf{THS} ($\uparrow$)
& \textbf{C} ($\uparrow$) & \textbf{H} ($\downarrow$)& \textbf{THS} ($\uparrow$)
& \textbf{C} ($\uparrow$) & \textbf{H} ($\downarrow$)& \textbf{THS} ($\uparrow$)
& \textbf{C} ($\uparrow$) & \textbf{H} ($\downarrow$)& \textbf{THS} ($\uparrow$)  \\ 

\midrule

\multicolumn{13}{l}{\textit{\textbf{Llama3.1-8B-Instruct}}} \\
\multicolumn{13}{l}{\textsc{Prompting Baseline}} \\
\quad Prompting     & 69.2 & 24.4 & 0.0 & 63.7 & 33.9 & 0.0 & 54.1 & 32.8 & 0.0 & 62.3 & 30.4 & 0.0 \\

\multicolumn{13}{l}{\textsc{Refusal-Aware SFT}} \\
\quad R-Tuning      & 56.9 & 12.7 & 20.9 & 64.2 & \underline{10.9} & 43.7 & 65.8 & 12.4 & 45.3 & 62.3 & 12.0 & 37.7 \\
\quad CRaFT         & 49.1 & 14.9 & 6.8 & 63.0 & \textbf{10.0} & 44.2 & 59.2 & \textbf{6.2} & 49.0 & 57.1 & \underline{10.4} & 35.8 \\

\multicolumn{13}{l}{\textsc{Confidence-Based Abstention}} \\
\quad CTSF          & 63.5 & 14.8 & 21.5 & 66.6 & 10.8 & 46.3 & 56.8 & 9.0 & 42.0 & 62.3 & 11.5 & 38.7 \\
\quad RLCR          & 60.6 & 35.8 & -40.9 & 59.3 & 40.0 & -15.9 & 50.6 & 41.5 & -17.8 & 56.8 & 39.1 & -23.3 \\
\quad GRPO + CTSF   & 50.6 & 46.7 & -81.8 & 68.0 & 28.3 & 14.8 & 58.8 & 30.7 & 8.2 & 59.1 & 35.2 & -13.0 \\

\multicolumn{13}{l}{\textsc{RLVR Methods}} \\
\quad GRPO          & \underline{86.2} & 12.2 & 51.6 & 76.1 & 22.9 & 33.1 & 79.7 & 13.1 & 58.1 & 80.7 & 16.1 & 47.7 \\
\quad TruthRL       & 85.3 & 11.6 & \underline{52.4} & 78.1 & 18.0 & 44.3 & \textbf{80.1} & 9.7 & \underline{64.1} & \underline{81.2} & 13.1 & \underline{54.4} \\

\multicolumn{13}{l}{\textsc{Factuality-Driven RL}} \\
\quad FSPO          & 77.7 & \underline{9.2} & 51.6 & 78.3 & 14.0 & 52.0 & 63.2 & 12.4 & 42.7 & 73.0 & 11.8 & 48.8 \\
\quad KnowRL        & 78.8 & 19.5 & 23.5 & \underline{81.9} & 16.7 & \underline{50.5} & 75.3 & 16.9 & 47.4 & 78.7 & 17.7 & 42.4 \\
\rowcolor{gray!15} \textbf{FaithRL} (Ours) & \textbf{87.5} & \textbf{9.1} & \textbf{61.7} & \textbf{85.5} & 11.4 & \textbf{64.1} & \underline{80.0} & \underline{8.8} & \textbf{65.5} & \textbf{84.3} & \textbf{9.8} & \textbf{64.2} \\

\midrule

\multicolumn{13}{l}{\textit{\textbf{Qwen2.5-7B-Instruct}}} \\
\multicolumn{13}{l}{\textsc{Prompting Baseline}} \\
\quad Prompting     & 71.5 & 13.1 & 0.0 & 72.7 & 19.3 & 0.0 & 59.2 & 16.3 & 0.0 & 67.8 & 16.2 & 0.0 \\

\multicolumn{13}{l}{\textsc{Refusal-Aware SFT}} \\
\quad R-Tuning      & 59.2 & 21.9 & -60.3 & 64.8 & 15.2 & 7.5 & 65.5 & 13.8 & 27.2 & 63.2 & 16.9 & -7.5 \\
\quad CRaFT         & 51.3 & 10.9 & -8.2 & 64.5 & 10.3 & 25.7 & 58.2 & \textbf{5.3} & 43.5 & 58.0 & \underline{8.8} & 21.2 \\

\multicolumn{13}{l}{\textsc{Confidence-Based Abstention}} \\
\quad CTSF          & 55.7 & 24.1 & -75.8 & 63.3 & 11.4 & 20.4 & 53.8 & 7.5 & 33.0 & 57.6 & 14.3 & -2.2 \\
\quad RLCR          & 64.9 & 27.0 & -82.5 & 67.1 & 29.7 & -44.8 & 52.0 & 29.1 & -53.7 & 61.3 & 28.5 & -58.0 \\
\quad GRPO + CTSF   & 64.8 & 29.2 & -94.6 & 76.1 & 15.3 & 18.5 & 63.5 & 13.2 & 26.9 & 68.1 & 19.2 & -12.3 \\

\multicolumn{13}{l}{\textsc{RLVR Methods}} \\
\quad GRPO          & \underline{84.9} & 13.1 & 13.4 & 84.2 & 14.8 & 28.4 & \textbf{79.4} & 15.3 & 37.0 & \textbf{82.8} & 14.4 & 22.5 \\
\quad TruthRL       & 83.8 & \underline{9.8} & \underline{30.3} & \underline{85.1} & \underline{10.2} & \underline{46.7} & \underline{77.7} & 8.0 & \underline{55.5} & 82.2 & 9.3 & \underline{43.3} \\

\multicolumn{13}{l}{\textsc{Factuality-Driven RL}} \\
\quad FSPO          & 72.4 & 14.9 & -8.9 & 67.2 & 26.1 & -31.1 & 58.6 & 13.9 & 8.1 & 66.0 & 18.3 & -10.6 \\
\quad KnowRL        & 73.8 & 23.9 & -5.7 & 79.7 & 18.9 & 8.5 & 73.8 & 17.6 & 9.9 & 75.8 & 20.1 & -8.3 \\
\rowcolor{gray!15} \textbf{FaithRL} (Ours) & \textbf{85.4} & \textbf{7.9} & \textbf{42.3} & \textbf{86.1} & \textbf{8.0} & \textbf{56.0} & 75.7 & \underline{5.9} & \textbf{59.3} & \underline{82.4} & \textbf{7.3} & \textbf{51.8} \\

\bottomrule
\end{tabularx}
\vspace{-1em}
\end{table*}

We substantiate the effectiveness of this modulation strategy through the following theorem:
\begin{theorem}[\textbf{Optimization Consistency and Bias Rectification}]
\label{thm:gradient_fidelity}
We prove that under the regime of strict faithfulness modulation ($\alpha = 0$), FaithRL achieves \textbf{optimization consistency}, ensuring that the expected policy gradient aligns strictly with the faithful objective $\mathcal{J}_C(\theta)$. Conversely, standard outcomes-based RL (vanilla RL) suffers from \textbf{process-outcome mismatch}, resulting in \textbf{biased optimization} by indiscriminately reinforcing spurious correctness and penalizing faltered reasoning.
\end{theorem}
The detailed proof is provided in Appendix \ref{app:proof_thm3}.

\section{Experiments}
\subsection{Experimental Setup}
\label{subsec:experimental_setup}

\paragraph{Datasets} We conduct evaluations on three multi-hop reasoning benchmarks: 2WikiMultiHopQA~\cite{ho2020constructing}, HotpotQA~\cite{yang2018hotpotqa}, and MuSiQue-Full~\cite{trivedi2022musique}. MuSiQue-Full naturally includes a mix of answerable and unanswerable questions. In contrast, the 2WikiMultiHopQA and HotpotQA datasets only contain answerable queries. To bridge this gap, we synthesize unanswerable samples for them, resulting in the 2WikiMultiHopQA-Full and HotpotQA-Full datasets.
Specifically, for each question $q$, we remove exactly one evidence document from $\mathcal{K}(q)$, which is the set of discrete documents. This manipulation ensures that the reasoning chain is logically broken, while the remaining knowledge retains high semantic relevance to $q$. This design strictly compels the model to determine answerability through genuine reasoning, effectively preventing it from exploiting superficial heuristics (e.g., keyword matching) or reward hacking to shortcut the judgment process.
We strictly filter these samples to ensure validity (see Appendix~\ref{subsec:detailed_datasets}). We train our model on the MuSiQue-Full dataset and evaluate performance on the test sets of all three \textit{full} benchmarks.


\paragraph{Metrics}
We adopt THS as our primary evaluation metric, as defined in Section~\ref{subsec:faithrl}. Additionally, we report the correctness rate $P(C)$ and hallucination rate $P(H)$ to provide a complete performance profile.

\vspace{-5pt}
\paragraph{Models and baselines}
We evaluate FaithRL against four categories of strong baselines:
\textbf{(1) refusal-aware SFT:} We select R-Tuning~\cite{zhang2024r} and CRaFT~\cite{zhu2025utilize}. These methods explicitly incorporate refusal data into instruction tuning to teach the model when to abstain.
\textbf{(2) confidence-based abstention:} We apply the CTSF thresholding strategy~\cite{jurayj2025your} to three distinct models: the original backbone, a GRPO-optimized model, and an RLCR-calibrated model~\cite{damani2025beyond}.
\textbf{(3) RLVR methods:} We include GRPO and TruthRL~\cite{wei2025truthrl}. These approaches optimize outcomes using sparse rewards, focusing on maximizing correctness and penalizing hallucinations.
\textbf{(4) factuality-driven RL:} We evaluate FSPO~\cite{lireasoning} and KnowRL~\cite{ren2025knowrl}. These methods utilize factuality-centric rewards to guide the reasoning process.

\subsection{Main Results}
Table \ref{tab:main_results} shows the overall experimental results, providing a comprehensive comparison between FaithRL and various baselines.

FaithRL consistently demonstrates superior performance in terms of THS across all settings. Compared with RLVR methods, FaithRL achieves an average improvement of \textbf{+9.0} points on the in-domain MuSiQue-Full dataset and \textbf{+18.5} points on the other two QA datasets. Specifically, FaithRL achieves a remarkable THS of 64.2\% using the Llama-3.1-8B-Instruct backbone, outperforming the strongest baseline, TruthRL (54.4\%), by a margin of 9.8 points. This advantage is particularly pronounced on HotpotQA-Full, where FaithRL leads the next best methods by absolute margins of 13.6 and 9.3 points on the Llama and Qwen backbones, respectively. These outcomes empirically support the efficacy of our geometric reward design in achieving an optimal equilibrium between helpfulness and honesty.

Beyond the aggregate THS metric, FaithRL achieves the highest average correctness while simultaneously maintaining the lowest hallucination rate. Compared with other RLVR methods, FaithRL achieves an average correctness improvement of \textbf{1.6} points and a hallucination reduction of \textbf{4.7} points. For instance, averaging across Llama-3.1-8B-Instruct tasks, FaithRL records 84.3\% correctness and a 9.8\% hallucination rate, surpassing GRPO with a 3.6\% increase in accuracy and a 6.3\% reduction in hallucinations. This Pareto optimality powerfully demonstrates that our method does not merely seek a trade-off between utility and honesty, but rather achieves a fundamental elevation in intrinsic model capability driven by faithful reasoning.

The performance of baseline methods reveals distinct characteristics. While refusal-aware SFT (e.g., CRaFT) and confidence-based abstention (e.g., CTSF) effectively curtail hallucinations compared to prompting baselines, they fail to yield significant gains in correctness. Conversely, RL methods like TruthRL and KnowRL enhance performance on the in-domain MuSiQue-Full dataset but exhibit diminished gains on other QA benchmarks. For example, using the Llama model, TruthRL experiences a THS drop of 11.7 and 19.8 points on 2WikiMultihopQA-Full and HotpotQA-Full respectively compared to in-domain performance. In contrast, FaithRL maintains stability with a negligible drop of 2.6 points. This highlights the robust generalization capabilities of FaithRL across diverse reasoning tasks. 
We provide more detailed experimental results in Appendix~\ref{subsec:detailed_main_results}.

\subsection{FaithRL Enhances Reasoning Faithfulness}
\label{subsec:faithfulness_analysis}

\paragraph{Dynamic Evolution of Reasoning Faithfulness during Training.}

\begin{figure}[t]
\begin{center}
\centerline{\includegraphics[width=0.45\textwidth]{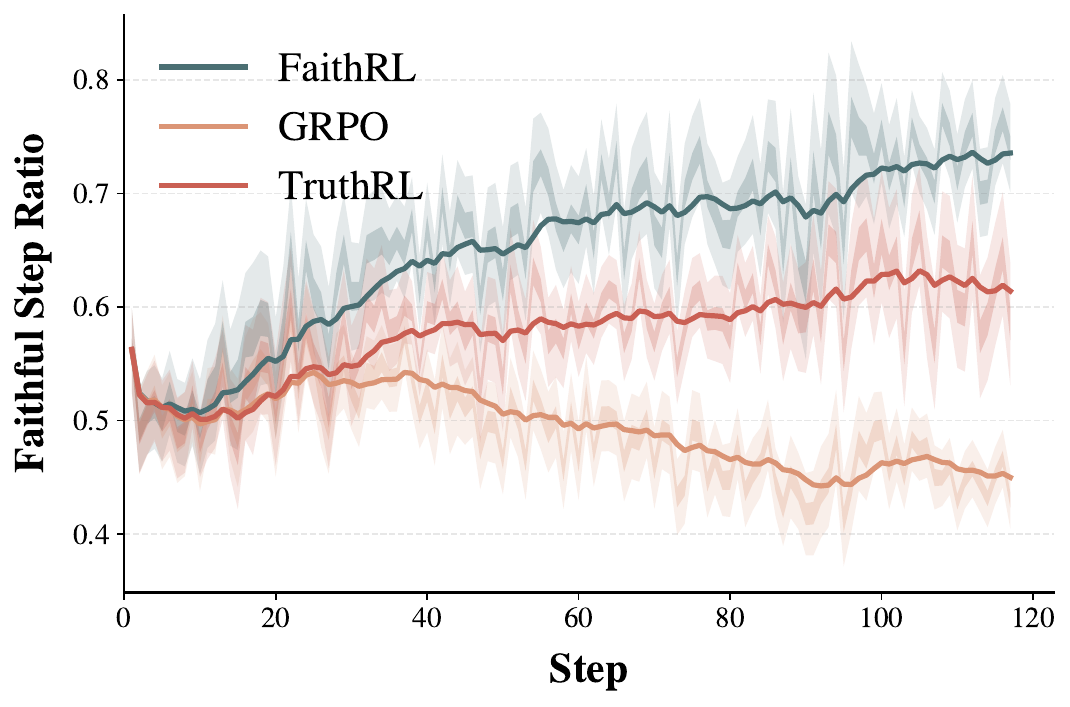}}
\caption{Comparison of the faithful step ratio among FaithRL, TruthRL, and GRPO during training.}
\vspace{-20pt}
\label{Figure5}
\end{center}
\end{figure}

We first investigate the ability of FaithRL to dynamically refine reasoning patterns throughout the optimization process. 
In this experiment, we train the Qwen-2.5-7B-Instruct model on the MuSiQue-Full dataset with 15,000 samples. For every training step, we monitor the reasoning trajectories of each sample and calculate the \textit{faithful step ratio}, defined as the proportion of reasoning steps strictly supported by evidence relative to the total steps.
As illustrated in Figure~\ref{Figure5}, our analysis reveals that FaithRL stands out as the only method capable of driving a significant and sustained increase in reasoning faithfulness. 
Specifically, the overall ratio of faithful steps increased from an initial baseline of 51.4\% to 82.5\% after training. 
In contrast, standard RL baselines, including GRPO and TruthRL, did not show meaningful improvements in faithfulness ratios. 
This disparity highlights the fundamental advantage of our approach. Unlike traditional outcome-oriented optimization, FaithRL provides precise step-level guidance to effectively shape the underlying reasoning process of the model.
We also provide a comprehensive breakdown of the training dynamics, categorized by sample correctness and step faithfulness, in Appendix~\ref{subsec:detailed_training}.

\begin{figure}[t] 
    \centering
    
    \begin{subfigure}[b]{0.49\linewidth} 
        \centering
        \includegraphics[width=\textwidth]{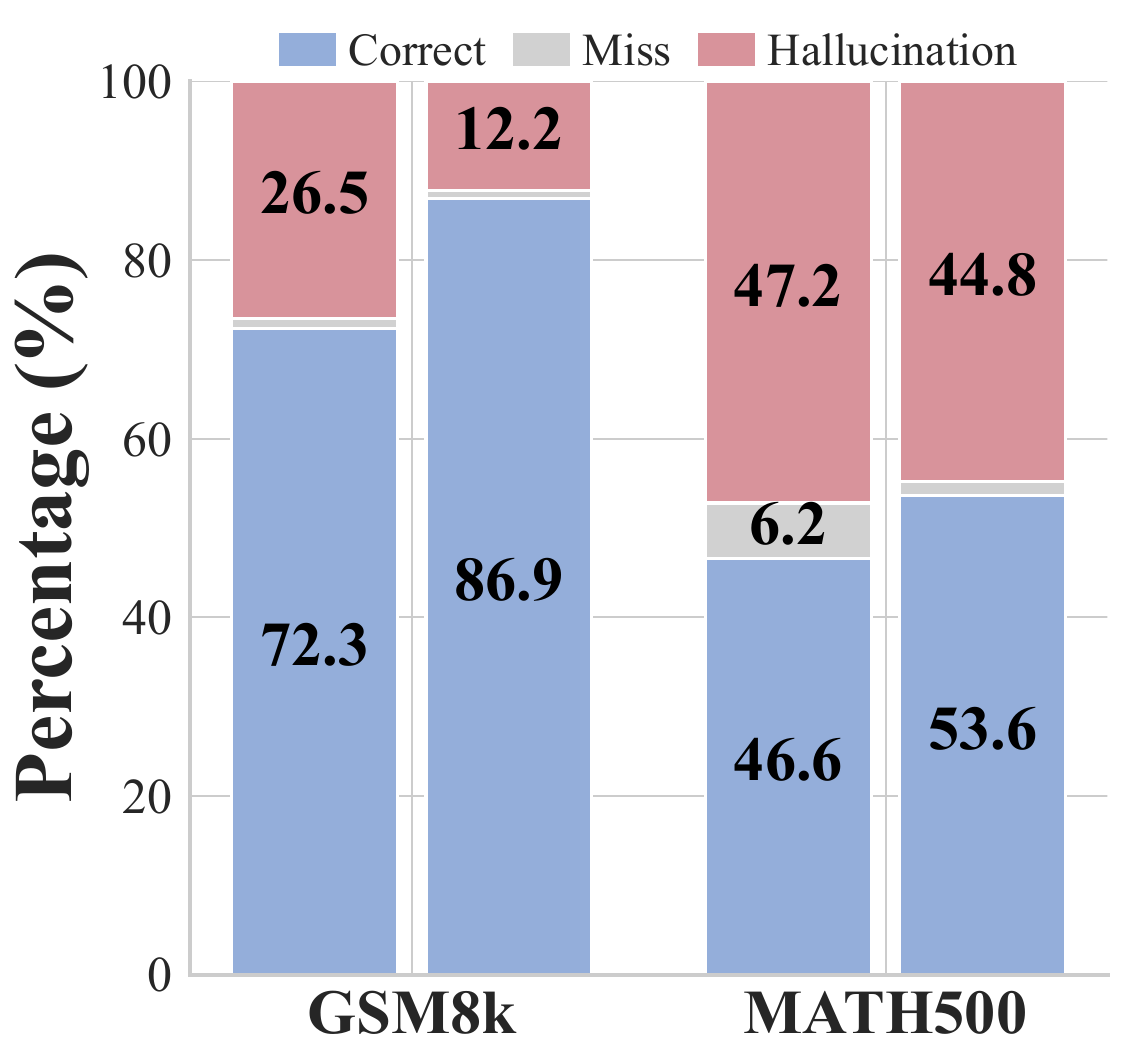}
        \caption{Outcome}
        \label{fig:outcome}
    \end{subfigure}
    \hfill 
    \begin{subfigure}[b]{0.49\linewidth}
        \centering
        \includegraphics[width=\textwidth]{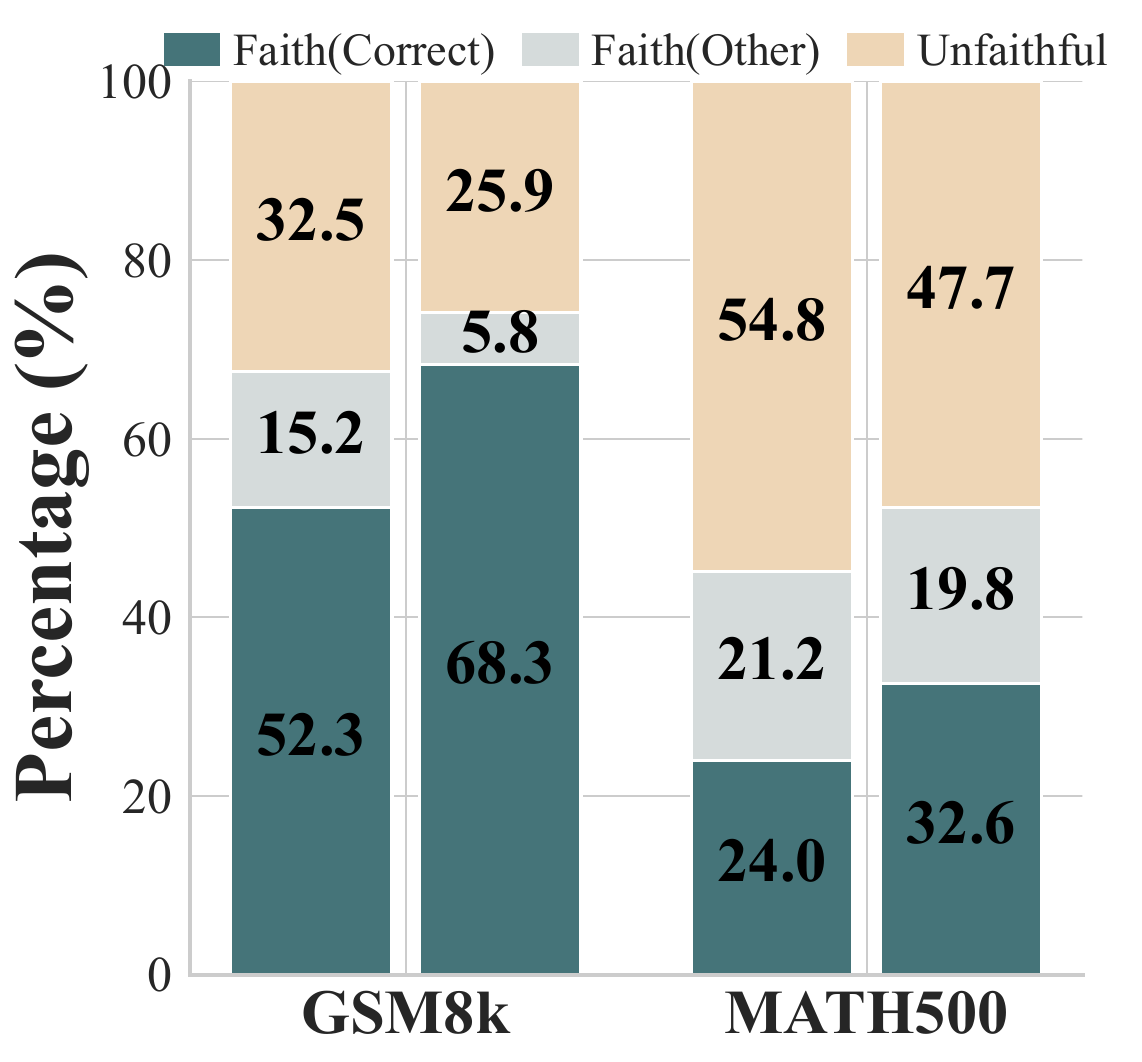}
        \caption{Process}
        \label{fig:process}
    \end{subfigure}
    
    \caption{Model performance on OOD tasks. (a) illustrates the outcome correctness distribution across datasets. (b) displays the faithfulness of the reasoning process. For each dataset, the \textbf{left} bar in each pair represents the GRPO baseline, and the \textbf{right} bar represents FaithRL. \textit{Faithful (Correct)} denotes reasoning steps that are verified as faithful within correctly answered samples.}
    \label{Figure6}
    \vspace{-10pt}
\end{figure}

\vspace{-5pt}
\paragraph{Generalization to out-of-distribution domains.}

To demonstrate that the faithfulness learned by FaithRL is a robust cognitive capability rather than a dataset-specific artifact, we evaluate the models on two Out-of-Distribution (OOD) benchmarks, GSM8k~\cite{cobbe2021training} and MATH500~\cite{hendrycks2021measuring}, comparing FaithRL against the GRPO baseline. To obtain faithfulness metrics, we employ an LLM to perform sentence-wise comparisons against the ground-truth solution processes. These evaluations are strictly inference-only, ensuring that no answer leakage or training signals are provided. Details regarding the experimental setup are provided in Appendix~\ref{app:training_details}.

As shown in Figure~\ref{Figure6}, FaithRL consistently outperforms GRPO. In terms of performance, FaithRL achieves significantly higher accuracy, surpassing GRPO by 14.6\% on GSM8k, while simultaneously maintaining a substantially lower hallucination rate. Regarding faithfulness, FaithRL exhibits a marked improvement in the overall faithful step ratio. Breakdown analysis shows this increase is primarily driven by improved faithfulness within correctly answered samples (rising from 52.3\% to 68.3\% on GSM8k). This suggests that the performance gains stem from genuine reasoning improvements rather than lucky guesses. We provide detailed results in Appendix~\ref{subsec:detailed_ood}.

\begin{figure}[t] 
    \centering
    \scriptsize 
    \begin{minipage}[c]{0.5\linewidth} 
        \centering
        \setlength{\tabcolsep}{2pt}   
        \renewcommand{\arraystretch}{1.2} 
        
        \begin{tabular}{lccc} 
            \toprule
            \textbf{Method} & \textbf{Cor.} & \textbf{Hal.} & \textbf{THS} \\
            \midrule
            Prompting & 67.8 & 16.2 & 0.0 \\
            GRPO & \textbf{82.8} & 14.4 & 22.5 \\ 
            + $\mathcal{R}_{geo}$ & 79.8 & 8.8 & 43.0 \\
            + FAAM & 77.8 & 12.0 & 27.5 \\ 
            \rowcolor{gray!15} \textbf{Ours} & 82.4 & \textbf{7.3} & \textbf{51.8} \\
            \bottomrule
        \end{tabular}
    \end{minipage}%
    \hfill 
    \begin{minipage}[c]{0.5\linewidth}
        \centering
        \includegraphics[width=\linewidth]{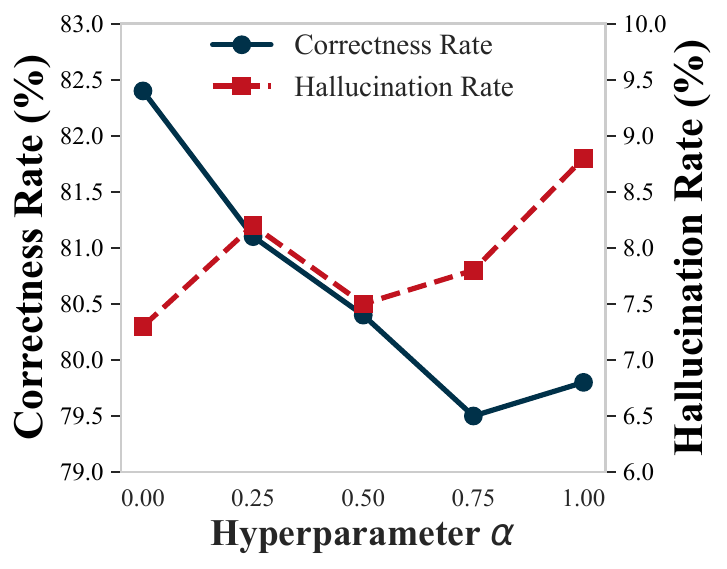} 
    \end{minipage}
    \caption{\textbf{Ablation studies.} (Left) Component analysis confirming the effectiveness of each module compared to baselines. (Right) Sensitivity analysis of $\alpha$, showing optimal performance at $\alpha=0$.}
    \label{fig:ablation_combined}
    \vspace{-20pt}
\end{figure}

\subsection{Ablations}
\label{subsec:ablations}

\paragraph{Component Analysis.} 
We verify the effectiveness of the two core components in FaithRL: the geometric reward ($\mathcal{R}_{geo}$) and the faithfulness-aware advantage modulation (FAAM). As detailed in \textbf{Figure~\ref{fig:ablation_combined} (Left)}, incorporating either $\mathcal{R}_{geo}$ or FAAM individually yields improvements in THS compared to the GRPO baseline. Specifically, adding $\mathcal{R}_{geo}$ alone results in a substantial THS gain of 20.5 points. However, deploying either component in isolation leads to a marked decline in correctness. In contrast, combining both modules enables FaithRL to achieve the lowest hallucination rate while maintaining high correctness, ultimately securing the highest THS. This demonstrates that $\mathcal{R}_{geo}$ and FAAM work synergistically to enforce faithful reasoning. For detailed experimental configurations and specific numerical results, please refer to Appendix~\ref{subsec:detailed_ablation_component}.

\vspace{-5pt}
\paragraph{Hyperparameter Sensitivity.} 
We further evaluate the sensitivity of our method to the hyperparameter $\alpha$ by varying it across $\{0, 0.25, 0.5, 0.75, 1.0\}$. The results, visualized in \textbf{Figure~\ref{fig:ablation_combined} (Right)}, show that decreasing $\alpha$ consistently mitigates hallucinations. In particular, $\alpha=0$ yields the optimal performance profile, effectively minimizing the hallucination rate compared to the vanilla RL baseline ($\alpha=1.0$) while preserving high correctness. We provide detailed results in Appendix~\ref{subsec:detailed_ablation_alpha}.

\vspace{-5pt}
\subsection{Computational Cost}
To quantify the computational overhead introduced by our method, we conducted a comparative analysis of FaithRL against GRPO using the Qwen-2.5-7B-Instruct backbone over one training epoch. We employed FLOPs and GPU hours as primary metrics to ensure a fair comparison of resource consumption across the entire training lifecycle. As shown in Figure~\ref{fig:efficiency}, empirical results indicate that despite FaithRL incorporating step-level supervision, the additional computational cost remains contained within 15\% relative to GRPO. This efficiency is primarily attributed to effective KV-cache reuse and the concurrent processing of verification requests. Detailed experimental setups and a theoretical complexity analysis are provided in Appendix \ref{app:comp_cost}.

\begin{figure}[t]
    \centering
    \begin{subfigure}[b]{0.49\linewidth}
        \centering
        \includegraphics[width=\textwidth]{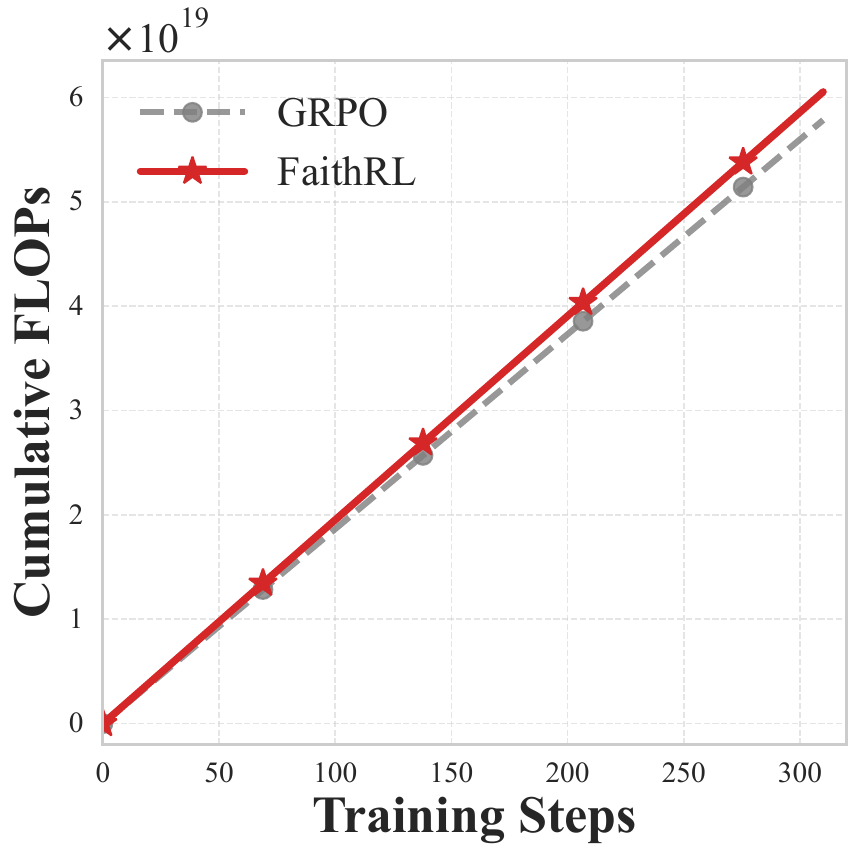}
    \end{subfigure}
    \hfill
    \begin{subfigure}[b]{0.49\linewidth}
        \centering
        \includegraphics[width=\textwidth]{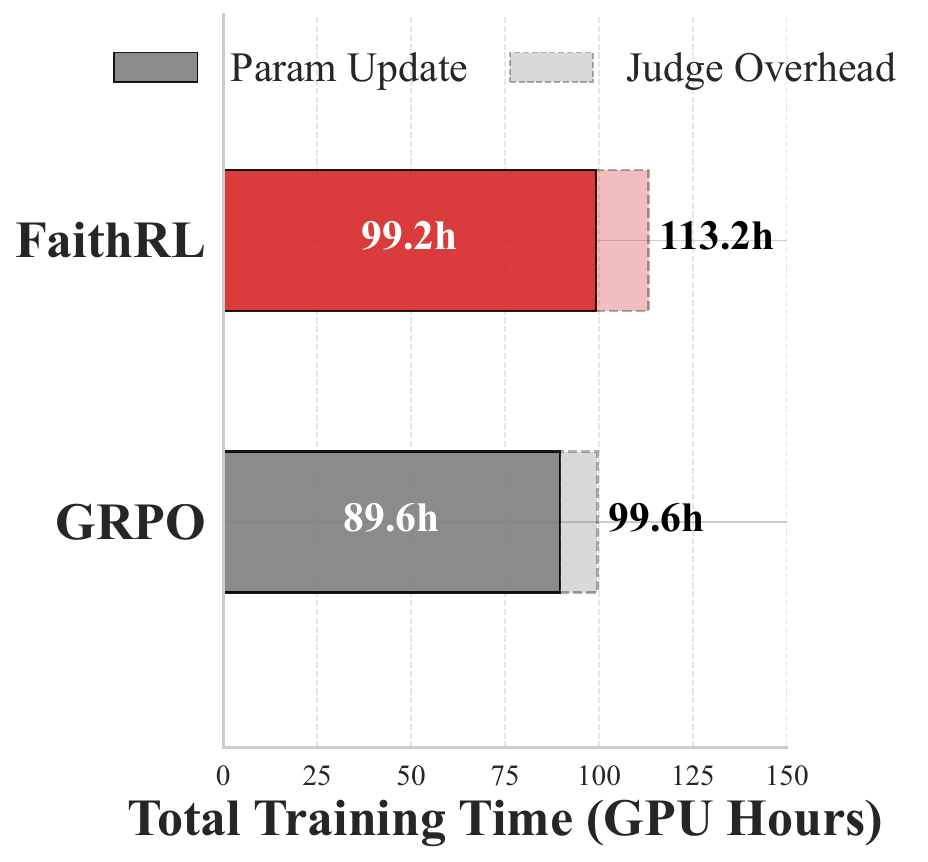}
    \end{subfigure}
    
    \caption{\textbf{Training efficiency analysis.} (Left) Comparison of cumulative FLOPs throughout training, showing that FaithRL incurs only a 4.7\% increase. (Right) Breakdown of total GPU hours, demonstrating that FaithRL limits the overhead to 13.7\%.}

    \label{fig:efficiency}
    \vspace{-15pt}
\end{figure}

\vspace{-5pt}
\section{Conclusion}
We presented FaithRL, a novel RL framework that enhances reasoning faithfulness by step-level verification. 
By strictly aligning model incentives with the validity of the reasoning process, our approach effectively bridges the gap between capability and honesty.
Our extensive experiments confirm that FaithRL not only achieves superior performance in terms of correctness and hallucination reduction but also successfully internalizes faithful reasoning patterns. 

\newpage

\section*{Impact Statement}
This paper presents work aimed at advancing the reliability and faithfulness of Large Language Models (LLMs). By introducing a framework that incentivizes reasoning grounded strictly in provided evidence, our approach addresses critical issues related to hallucinations and over-confidence in AI systems. The potential societal impact is primarily positive, as it fosters the development of more trustworthy AI assistants suitable for deployment in high-stakes domains such as healthcare, law, and education, where factual accuracy is non-negotiable. Furthermore, by promoting transparent and verifiable reasoning processes, this work supports the broader goal of interpretable AI, helping to mitigate the spread of misinformation generated by confident but incorrect models. While the training process incurs a marginal increase in computational resources, the long-term benefits of deploying safer, hallucination-resistant models likely outweigh these costs.

\bibliography{example_paper}
\bibliographystyle{icml2026}

\newpage
\appendix
\onecolumn 

\section{Limitations}
Despite the promising results, we acknowledge two primary limitations. First, FaithRL relies on datasets with process-level annotations, such as evidence chains. While our OOD experiments indicate that the acquired reasoning capabilities are highly transferable, reducing the dependency on granular supervision remains an important direction for future work. Second, the integration of an external verifier introduces additional computational overhead. However, empirical analysis confirms that this cost is acceptable (within 15\%) given the substantial gains in reliability. We leave it to future work to further optimize the verification efficiency.

\section{Experimental Validation of Assumptions}
\label{app:assumption_validation}

To empirically validate the soundness of the theoretical assumptions posited in Section 3, we conduct a rigorous analysis on the MuSiQue-Full dataset ($N=39,876$). To fully capture the dynamic satisfaction of these assumptions throughout the training lifecycle, we evaluate both the initial pre-trained model Qwen-2.5-7B-Instruct and the final \textbf{FaithRL}-optimized model derived from it.

We design two experimental protocols to verify the causal links between evidence sufficiency, reasoning faithfulness, and outcome validity. For \textbf{Assumption \ref{asm:iff_evidence} (Correctness Sufficiency)}, we employ an LLM Judge to assess evidence utilization, isolating trajectories that fully incorporate the complete evidence set $\mathcal{E}(q)$ to measure answer accuracy. For \textbf{Assumption \ref{asm:faithfulness} (Hallucination Prevention)}, we evaluate the logical consistency of reasoning steps against the provided context, filtering for strictly faithful trajectories to measure the hallucination rate. Llama-3.3-70B-Instruct serves as the unified judge. For inference, we set the temperature to 0 and the maximum output length to 4096. All prompt templates used for evaluation are detailed in Appendix~\ref{app:prompt_template}.

The results are summarized in Table \ref{tab:assumption_stats}.

\begin{table}[h]
    \centering
    \caption{Statistical Validation of Theoretical Assumptions on MuSiQue-Full.}
    \label{tab:assumption_stats}
    \begin{tabular}{lcccc}
        \toprule
        \textbf{Model} & \textbf{Assumption} & \textbf{Condition Met} & \textbf{Result Satisfied} & \textbf{Ratio} \\
        \midrule
        \multirow{2}{*}{Qwen-2.5-7B-Inst} & Asm. \ref{asm:iff_evidence} (Correctness) & 10,288 & 7,497 & 72.9\% \\
         & Asm. \ref{asm:faithfulness} (Hallucination) & 14,355 & 12,840 & 89.5\% \\
        \midrule
        \multirow{2}{*}{FaithRL (Ours)} & Asm. \ref{asm:iff_evidence} (Correctness) & 15,392 & 14,036 & 91.2\% \\
         & Asm. \ref{asm:faithfulness} (Hallucination) & 10,208 & 9,809 & 96.1\% \\
        \bottomrule
    \end{tabular}
\end{table}

As demonstrated in Table \ref{tab:assumption_stats}, the empirical results strongly support our theoretical framework. For the pre-trained model, the adherence ratios are 72.9\% and 89.5\%. Post-training with FaithRL, these ratios improve significantly to 91.2\% and 96.1\%, respectively. Averaging these states to approximate the conditions during training, we find a mean Correctness Sufficiency ratio of \textbf{82.0\%} and a mean Hallucination Prevention ratio of \textbf{92.8\%}. These high aggregate metrics confirm that our theoretical assumptions hold statistically significant validity within the domain of knowledge-intensive reasoning.

\section{Proofs of Theorem 1}
\label{app:proofs}
Let $\Pi$ be the space of stochastic policies parameterized by $\theta$. For a given input $q$, $\pi_\theta$ induces a distribution over trajectories $\mathcal{T}$. Consistent with Section 4.1, the sample space is partitioned into disjoint sets: $S_{c}$ (Correct), $S_{m}$ (Miss), $S_{h}$ (Hallucination).
We analyze the asymptotic behavior of the Miss Rate $P(M) = P(\tau \in S_m)$. Note that $S_m$ is only defined for answerable queries ($\mathcal{Q}_{ans}$), as refusal is correct for $\mathcal{Q}_{unans}$.

\subsection{Analysis of Objective A: Maximize Correctness}

\textbf{Proposition:} \textit{Optimizing solely for correctness ($J_A(\theta) = P_\theta(S_c)$) drives the Miss Rate $P(M)$ asymptotically to 0.}

\begin{proof}
Consider the reward structure: $r(\tau) = 1$ if $\tau \in S_{c}$, and $0$ otherwise.
For any answerable query $q \in \mathcal{Q}_{ans}$:
\begin{itemize}
    \item \textbf{Action: Refuse ($\tau \in S_m$).} By definition of Miss, the answer is IDK but $A^* \neq \text{IDK}$. Thus $\tau \notin S_c$. Reward $r=0$.
    \item \textbf{Action: Attempt.} The model generates a candidate answer. Even if the model is uncertain, there exists a non-zero probability that the generated answer matches $A^*$ (either through knowledge or lucky guessing). Thus, the expected reward $E[r|\text{Attempt}] = P(\text{Correct}|\text{Attempt}) > 0$.
\end{itemize}
Since $E[r|\text{Attempt}] > E[r|\text{Refuse}]=0$, the policy gradient consistently pushes the probability mass away from refusal ($S_m$) towards attempting. Asymptotically, $P(M) \to 0$.
\textbf{Consequence:} This creates Over-Confidence. The model learns that "trying" is always strictly better than "missing", encouraging hallucination on difficult or unanswerable queries (where it might also get lucky).
\end{proof}

\subsection{Analysis of Objective B: Minimize Hallucination}

\textbf{Proposition:} \textit{Optimizing solely to minimize hallucination ($J_B(\theta) = -P_\theta(S_h)$) drives the Miss Rate $P(M)$ asymptotically to 1 (Degenerate Solution).}

\begin{proof}
The objective is to minimize $P(H)$, or equivalently maximize likelihood of Non-Hallucination ($S_c \cup S_m$).
For any answerable query $q \in \mathcal{Q}_{ans}$:
\begin{itemize}
    \item \textbf{Action: Refuse ($\tau \in S_m$).} By definition, $S_m$ involves outputting "IDK". Since "IDK" is not in the Hallucination set $\mathcal{H}$, the risk of hallucination is strictly 0.
    \item \textbf{Action: Attempt.} Due to model imperfection, any attempt to answer carries a risk $\epsilon > 0$ of generating an incorrect answer ($\tau \in S_h$).
\end{itemize}
To strictly minimize the hallucination rate to 0, the optimal policy is to choose the action with 0 risk. Refusal ($S_m$) is a "safe haven" that guarantees zero penalty under Objective B.
Thus, the gradient favors shifting probability mass from Attempt to Refuse. Asymptotically, the model learns to refuse all queries to achieve perfect safety, driving $P(M) \to 1$.
\textbf{Consequence:} This creates Over-Conservativeness. Usefulness is sacrificed for safety.
\end{proof}

\subsection{Analysis of Objective C: Maximize Adherence}

\textbf{Proposition:} \textit{Maximizing adherence ($J_C(\theta) = P_\theta(S_f)$) stabilizes the Miss Rate $P(M)$.}

\begin{proof}
Objective C rewards membership in $S_f = \{ \tau \mid \phi(\tau) = \mathcal{E}(q) \}$.
We analyze the gradient direction for answerable queries $q \in \mathcal{Q}_{ans}$ (where $S_m$ exists):

\begin{itemize}
    \item \textbf{Refusal ($S_m$):} If the model refuses, it outputs "IDK". However, for $q \in \mathcal{Q}_{ans}$, the evidence set $\mathcal{E}(q)$ consists of specific knowledge facts required to derive the answer. A refusal trajectory generally fails to utilize/cite these specific facts (or cites them but fails to derive the answer). Thus $\phi(\tau) \neq \mathcal{E}(q)$, so $\tau \notin S_f$. Reward $r=0$.
    \item \textbf{Faithful Attempt ($S_f$):} If the model correctly uses evidence $\mathcal{E}(q)$ to derive $A^*$, then $\tau \in S_f$. Reward $r=1$.
\end{itemize}
For these queries, Objective C behaves like Objective A: it penalizes Misses ($r=0$) and rewards Faithful Attempts ($r=1$). Thus, it exerts a downward pressure on $P(M)$, preventing the Over-Conservativeness of Objective B.

However, unlike Objective A, Objective C does not drive the model to speculative extremes. For queries where the model cannot find evidence (e.g., due to capability limits or unanswerable nature), it cannot satisfy the condition $\phi(\tau) = \mathcal{E}(q)$. In such cases, attempting to answer often leads to $S_{extra}$ (spurious) or $S_h$, which are not rewarded ($r=0$) or penalized.
By strictly tying reward to evidence presence, Objective C ensures that $P(M)$ is reduced \textit{only} when the model is capable of faithful reasoning, thereby stabilizing the refusal strategy at a level consistent with the model's actual capability.
\end{proof}

\section{Metric Details: Truthful Helpfulness Score (THS)}
\label{app:ths_details}

Due to the shortcomings of existing metrics, we employ the Truthful Helpfulness Score (THS). 

\paragraph{Geometric Definition.}
To rigorously ground the definition of THS, we consider the polar characteristics of the performance plane. The angle $\theta$ of a capability vector relative to the positive x-axis, representing the Correctness Rate, inversely correlates with model quality. A smaller angle implies a lower ratio of hallucinations to correct answers. The ray $\vec{OE_0}$ thus establishes an "iso-capability" threshold. Geometrically, the perpendicular distance of a new state $E_1$ from this ray serves as a direct measure of capability enhancement. Since the area of $\triangle OE_0E_1$ is proportional to this altitude given a fixed base $\vec{OE_0}$, the signed area naturally quantifies the magnitude of improvement. Points to the right of the ray yield a positive area, indicating capability gain, while those to the left yield a negative area, signaling degradation.

Operationally, we calculate this signed area using the cross product in 2D space. Let $E_0 = (x_0, y_0)$ and $E_1 = (x_1, y_1)$. The THS is computed as the ratio of the achieved signed area to the maximum potential area:
\begin{equation}
    \text{THS} = \frac{\vec{OE_1} \times \vec{OE_0}}{\vec{OE^*} \times \vec{OE_0}} = \frac{x_1 y_0 - x_0 y_1}{1 \cdot y_0 - x_0 \cdot 0} = \frac{x_1 y_0 - x_0 y_1}{y_0}
\end{equation}
This formulation ensures that the metric is strictly monotonic with respect to the geometric distance from the baseline.

As referenced in Figure~\ref{Figure4}, the area of the triangle $\triangle OE_0E_1$ represents the magnitude of capability improvement. The THS effectively normalizes this area, providing a metric with a theoretical upper bound of 1, where positive values indicate improvement over the baseline.

\paragraph{Effectiveness Analysis.}
To illustrate the necessity of THS, consider the specific example involving two models: Model A with $P_c=0.7, P_h=0.1$ and Model B with $P_c=0.8, P_h=0.2$. While Model B has higher accuracy, its hallucination rate is double that of Model A. By establishing Model A as the baseline $E_0$, Model B yields a THS of -60\%, clearly indicating a negative impact on trustworthiness despite the accuracy gain.

We further validate the metric through controlled variable analysis to demonstrate its robustness. When the Hallucination Rate is held constant, an increase in the Correctness Rate naturally expands the signed area of the capability triangle, resulting in a higher THS. Conversely, when the Correctness Rate is fixed, a reduction in the Hallucination Rate similarly increases the signed area and improves the score. This confirms that maximizing THS is theoretically consistent with the dual goals of enhancing utility and reducing hallucination, without conflict.

For more comprehensive details, please refer to Appendix C.3 of CRaFT~\cite{zhu2025utilize}.

\section{Proof of Theorem 2}
\label{app:proof_thm2}

\begin{proof}
Let the performance of the model be represented in a 2D Cartesian plane where the x-axis denotes the Correctness Rate and the y-axis denotes the Hallucination Rate.
We define the baseline capability vector as $\mathbf{v}_0 = (x_0, y_0)$ and the current policy's performance vector as $\mathbf{v}_\pi = (P_\pi(C), P_\pi(H))$.

\textbf{1. Geometric Interpretation of the Objective}
The Truthful Helpfulness Score (THS) measures the capability gain as the ratio of the signed area swept by the performance vector relative to the baseline. 
The signed area of the triangle formed by the origin $O$, the baseline $\mathbf{v}_0$, and the current state $\mathbf{v}_\pi$ is given by half the magnitude of the 2D cross product:
\begin{equation}
    \text{Area}(O, \mathbf{v}_0, \mathbf{v}_\pi) = \frac{1}{2} (\mathbf{v}_\pi \times \mathbf{v}_0) = \frac{1}{2} (P_\pi(C) \cdot y_0 - P_\pi(H) \cdot x_0)
\end{equation}
By definition, THS is the normalized area relative to the ideal area $S_{ideal}$ (determined by the perfect state $\mathbf{v}^*$):
\begin{equation}
    \text{THS}(\pi_\theta) = \frac{\text{Area}(O, \mathbf{v}_0, \mathbf{v}_\pi)}{S_{ideal}}
\end{equation}
Crucially, $S_{ideal}$ is determined solely by the baseline $\mathbf{v}_0$ and the task ground truth, making it a constant positive scalar with respect to the policy parameters $\theta$.

\textbf{2. Formulation of the Expected Return}
The expected return $J(\theta)$ under our proposed reward scheme $r(\tau)$ is calculated by summing the probability-weighted rewards for each outcome set $S_c, S_m, S_h$:
\begin{align}
    J(\theta) &= \mathbb{E}_{\tau \sim \pi_\theta}[r(\tau)] \\
    &= P_\pi(S_c) \cdot (+y_0) + P_\pi(S_m) \cdot 0 + P_\pi(S_h) \cdot (-x_0) \\
    &= y_0 P_\pi(C) - x_0 P_\pi(H)
\end{align}

\textbf{3. Equivalence and Gradient Alignment}
Comparing the expressions, we observe that the expected return is exactly double the signed area:
\begin{equation}
    J(\theta) = 2 \cdot \text{Area}(O, \mathbf{v}_0, \mathbf{v}_\pi)
\end{equation}
Substituting the definition of THS:
\begin{equation}
    J(\theta) = (2 S_{ideal}) \cdot \text{THS}(\pi_\theta)
\end{equation}
Since $S_{ideal} > 0$, the coefficient $C = 2 S_{ideal}$ is a strictly positive constant. Calculating the gradient with respect to $\theta$ yields:
\begin{equation}
    \nabla_\theta J(\theta) = C \cdot \nabla_\theta \text{THS}(\pi_\theta)
\end{equation}
This establishes a strict positive proportionality (collinearity). It implies that the gradient vector of our reward function points in the exact same direction as the gradient of the Truthful Helpfulness Score, ensuring identical optimization paths.
\end{proof}

\section{Proof of Theorem 3}
\label{app:proof_thm3}

\textbf{Assumption:} We analyze the theoretical properties under the strict faithfulness enforcement regime ($\alpha=0$). Consistent with Theorem 1, we partition the trajectory space $\mathcal{T}$ based on the alignment between reasoning quality ($\mathcal{V}$) and outcome correctness ($r$). We identify two critical subsets of \textit{Process-Outcome Mismatch}:

\begin{itemize}
    \item \textbf{Spurious Set (Lucky Guesses)} $S_{spur} = \{ \tau \mid \psi(\tau)=A^* \land \mathcal{V}(\tau)=0 \}$. The answer is correct, but reasoning is unfaithful.
    \item \textbf{Faltered Set (Incidental Failures)} $S_{fail} = \{ \tau \mid \psi(\tau) \neq A^* \land \mathcal{V}(\tau)=1 \}$. The reasoning is faithful (evidence covered), but the final answer is incorrect (e.g., calculation error or token generation slip).
\end{itemize}
The remaining sets are the \textbf{Faithful Set} $S_f$ (Both correct) and the \textbf{Hallucination Set} $S_h$ (Both wrong).

\paragraph{1. Failure Analysis of Vanilla RL (Dual Bias).}
Consider a standard RL baseline optimizing solely for correctness (Objective A). The expected gradient is approximately:
\begin{equation}
    \nabla \mathcal{J}_{Vanilla} \propto \mathbb{E}_{\tau} [r(\tau) \nabla \log \pi(\tau)]
\end{equation}
This leads to two types of optimization errors:
\begin{enumerate}
    \item \textbf{Positive Bias (Reinforcing Bad Logic):} For $\tau \in S_{spur}$, $r(\tau)=1$. The model is encouraged to repeat flawed reasoning or shortcuts.
    \item \textbf{Negative Bias (Suppressing Good Logic):} For $\tau \in S_{fail}$, $r(\tau) \le 0$ (typically 0 or negative). The model is discouraged from performing the correct reasoning steps simply because the final token was wrong.
\end{enumerate}
Thus, 
\begin{align*}
    \nabla \mathcal{J}_{Vanilla} \propto \nabla P(S_f) + \underbrace{\nabla P(S_{spur})}_{\text{Noise}} - \underbrace{\nabla P(S_{fail})}_{\text{Damage}}.
\end{align*}

\paragraph{2. Dual Rectification in FaithRL.}
FaithRL introduces the modulation term $M(\tau)$ (via FAAM). We analyze how it handles the two mismatch cases under $\alpha = 0$:

\textbf{Case A: Spurious Correctness ($S_{spur}$)}
\begin{itemize}
    \item Outcome is positive ($A > 0$), but Verifier is negative ($\mathcal{V}=0$).
    \item Modulation: $M = (1-0)\cdot 0 + 0 = 0$.
    \item Effect: The reward is filtered out. $\nabla \mathcal{J} \leftarrow 0$. The model \textbf{does not learn} from lucky guesses.
\end{itemize}

\textbf{Case B: Faltered Reasoning ($S_{fail}$)}
\begin{itemize}
    \item Outcome is negative ($A \le 0$), but Verifier is positive ($\mathcal{V}=1$).
    \item Modulation: $M = (1-\alpha)(1-\mathcal{V}) + \alpha$. Here, for negative samples, we want to penalize invalid steps.
    \item Since $\mathcal{V}=1$, $M = (1-0)(0) + 0 = 0$.
    \item Effect: The penalty is filtered out/mitigated. $\nabla \mathcal{J} \leftarrow 0$. The model \textbf{is not punished} for correct reasoning chains that accidentally failed at the end.
\end{itemize}

\paragraph{Conclusion.}
Combining these effects, the FaithRL gradient becomes:
\begin{equation}
    \nabla \mathcal{J}_{FaithRL} \propto y_0 \nabla P(S_f) - x_0 \nabla P(S_{h}) + 0 \cdot \nabla P(S_{spur}) - 0 \cdot \nabla P(S_{fail})
\end{equation}
By effectively zeroing out the contributions from both types of mismatch, FaithRL ensures the optimization is purely driven by the alignment of process and outcome ($S_f$ and $S_h$), proving strict consistency with the faithful objective.

\section{Experimental Details}

\subsection{Datasets}
\label{subsec:detailed_datasets}

\subsubsection{Dataset Overview}
We evaluate our framework on three multi-hop QA benchmarks selected for their distinct complexity profiles. 

\textbf{HotpotQA} is a large-scale multi-hop QA dataset with 112,779 Wikipedia-based question-answer pairs, constructed via crowdsourcing. Each sample includes sentence-level supporting facts to enable explainable reasoning, and all questions require multi-hop inference across multiple documents. It features diverse questions (covering entities, dates, locations, etc.) and a novel type of factoid comparison questions (including yes/no and numerical comparison tasks), without being constrained by pre-existing knowledge bases or schemas, supporting comprehensive evaluation of models’ multi-hop reasoning and explainability.

\textbf{2WikiMultiHopQA} is a large-scale multi-hop QA dataset with 192,606 samples, constructed by integrating unstructured text from Wikipedia and structured triples from Wikidata. Each sample includes explicit evidence in the form of structured triples to illustrate the reasoning path, ensuring all questions require multi-hop inference to answer. It covers four question types (comparison, inference, compositional, bridge comparison) generated via predefined templates and logical rules, facilitating comprehensive evaluation of models’ reasoning capabilities.

\textbf{MuSiQue-Full} is a challenging multi-hop QA dataset with ~25K 2–4 hop questions, constructed via a bottom–up approach by composing single-hop questions from 5 Wikipedia-based datasets. It enforces strict connected reasoning (each step relies on previous outputs) through stringent filters, minimizing train-test leakage and including hard-to-identify distractors. Featuring six reasoning graph structures, it has a 3× larger human–machine gap and lower cheatability (single-hop models drop 30 F1 points) than existing datasets, with natural question decompositions to support interpretable reasoning model development.

\subsubsection{Evidence Construction}
\label{app:evidence_construction}
For the MuSiQue-Full dataset, each sample includes a `question\_decomposition` field representing prior knowledge. We leverage this structured data to construct the target evidence set $\mathcal{E}(q)$. For example, a raw decomposition might appear as:
\begin{verbatim}
[{"id": 42543, "question": "who wrote crazy little thing called love original 
  artist", "answer": "Freddie Mercury"}, 
 {"id": 20093, "question": "In what year did #1 die?", "answer": "1991"]
\end{verbatim}
To transform these pairs into a coherent evidence set, we utilize Llama-3.3-70B-Instruct with the prompt defined in Appendix \ref{app:prompt_template}. The model resolves references (e.g., ``\#1'') and synthesizes declarative statements, resulting in the final $\mathcal{E}(q)$:
\begin{quote}
['Freddie Mercury wrote crazy little thing called love original artist', 'Freddie Mercury died in year 1991.']
\end{quote}

\subsubsection{Dataset Construction for 2WikiMultihopQA-Full and HotpotQA-Full}
\label{app:dataset_construction}
To construct the unanswerable subsets for our Out-of-Domain test sets (2WikiMultiHopQA-Full and HotpotQA-Full), we employ a two-stage process consisting of strategic pruning and rigorous filtering.

\paragraph{Phase 1: Evidence Pruning}
Both datasets provide a \texttt{supporting\_facts} field for each query $q$, formatted as a list of indices (e.g., \texttt{[["Arthur's Magazine", 0], ["First for Women", 0]]}) that point to specific sentences within the context documents. We utilize this ground truth to identify the complete evidence set $\mathcal{E}(q)$ within the retrieved context $\mathcal{K}(q)$. The initial generation of unanswerable samples involves randomly pruning documents referenced by these indices, thereby creating reasoning gaps.

\paragraph{Phase 2: Rigorous Quality Control}
To ensure the constructed samples are genuinely unanswerable due to missing logic rather than trivial irrelevance, and to prevent model shortcuts (reward hacking), we enforce three strict technical constraints during a secondary filtering stage:

\begin{enumerate}
    \item \textbf{Context Density Maintenance:} If the number of remaining evidence documents in $\mathcal{E}(q)$ exceeds 3 after the initial pruning, we iteratively prune an additional document. This ensures that the remaining context $\mathcal{K}(q)$ maintains high semantic density and relevance to the query, preventing the model from identifying unanswerable questions simply by detecting a lack of relevant keywords (a common form of reward hacking).
    
    \item \textbf{Hop-Sensitive Pruning:} We strictly prohibit the pruning of documents corresponding to the first hop of the reasoning chain. Instead, we randomly target intermediate or final hops. This is because first-hop entities are often explicitly mentioned in the question itself, making them easily recoverable or hallucinated. Removing later hops ensures the break in reasoning occurs deep within the logic chain.
    
    \item \textbf{Empirical Solvability Verification:} We employ a verification step using \texttt{Qwen2.5-7B-Instruct} with a \textit{Best-of-32} inference strategy. For every generated unanswerable candidate, we sample 32 reasoning paths. If the model succeeds in deriving the correct answer in \textit{any} of the 32 attempts (indicating the reasoning gap was bypassed via hallucination or leakage), the sample is discarded. We retain only those samples where all 32 attempts fail, ensuring the target output is strictly ``I don't know''.
\end{enumerate}

\subsection{Implementation Details of the Baselines}
\label{app:baselines}
Below, we provide implementation details for the baseline methods:

\begin{itemize}
    \item \textbf{R-Tuning:} We apply R-Tuning~\cite{zhang2024r} on the MuSiQue-Full training set using the backbone model for inference. We append "I am sure" to the model's response for correct reasoning trajectories and "I am unsure" for incorrect ones. The backbone model is then supervised fine-tuned (SFT) on this augmented dataset using full-parameter updates with a cosine learning rate schedule for 3 epochs. During evaluation, responses containing "I am unsure" are mapped to "IDK", while answers are extracted from those marked "I am sure".
    
    \item \textbf{CRaFT:} Following \citet{zhu2025utilize}, we first sample 10 responses for each training query to compute accuracy and certainty scores. We select samples in the top 10\% of both metrics for Rehearsal Training to obtain a rehearsal model. This model is then employed for a two-stage data filtering process, and the final model is obtained by SFT on the filtered dataset.
    
    \item \textbf{CTSF:} We implement the Confidence-Based Thresholding strategy~\cite{jurayj2025your} using the cumulative log-probability of the generated sequence as the confidence metric. We set the refusal threshold to -15; responses falling below this value are converted to "IDK".
    
    \item \textbf{RLCR:} We calibrate the initial model via RL using the reward function $R = 1 - (cr - conf)^2$, where $cr \in \{0, 1\}$ denotes correctness (1 for Correct, 0 for Error/Miss) and $conf \in [0, 1]$ is extracted from a generated \texttt{<confidence>} tag~\cite{damani2025beyond}. Post-training, we use a fixed confidence threshold of 0.5 to determine refusal.
    
    \item \textbf{TruthRL:} We train the model using PPO with the ternary reward scheme proposed by \citet{wei2025truthrl}, assigning rewards of +1 for correct answers, -1 for hallucinations, and 0 for other outcomes.
    
    \item \textbf{FSPO:} We fully reproduce the original implementation by directly adopting the official code of Fine-Grained Step-Level Policy Optimization~\cite{lireasoning}, with only the train batch size adjusted to 128 to enhance the stability and reproducibility of the training process.
    
    \item \textbf{KnowRL:} Following \citet{ren2025knowrl}, we design a composite reward function comprising format, factuality, and correctness components. We set the weighting coefficients for these three terms to a ratio of 1:1:1.
\end{itemize}

\subsection{Training Details}
\label{app:training_details}
\paragraph{Implementation \& Hardware.} We implement our framework using the \texttt{verl} library on a computational node equipped with 4 NVIDIA H200 GPUs (141GB VRAM each). Models undergo full-parameter fine-tuning with a global training batch size of 128 and a validation batch size of 512. Maximum sequence lengths are set to 8192 tokens for prompts and 2048 tokens for responses.

\paragraph{Optimization Hyperparameters.} Optimization utilizes a cosine learning rate scheduler with a peak learning rate of $1 \times 10^{-6}$, a warmup ratio of 0.1, and a minimum learning rate ratio of 0.1. For PPO training, we set the mini-batch size to 64 and the micro-batch size per GPU to 16. The KL divergence penalty is disabled (\texttt{use\_kl\_loss=False}), and the rollout sampling temperature is set to 1.0.

\paragraph{Oracle \& Verification.} To accurately categorize outcomes into Correct ($S_c$), Miss ($S_m$), and Hallucination ($S_h$) as required by our geometric reward, we implement a hierarchical evaluation protocol. We first employ rigorous string matching to detect explicit refusals (e.g., ``I don't know''), classifying such instances as Miss ($S_m$) in answerable queries and Correctness ($S_c$) in unanswerable queries. For non-refusal responses, \texttt{Llama-3.3-70B-Instruct} serves as the oracle judge to distinguish between Correct ($S_c$) and Hallucination ($S_h$). This model also conducts step-wise evidence verification ($\mathcal{V}$), where we segment the reasoning process into steps based on sentence boundaries. We select a strong model for this verification primarily to leverage shared context caching with the outcome judge, thereby minimizing additional overhead (see Section~\ref{app:comp_cost}). Crucially, this task is distinct from reasoning distillation; it constitutes an objective check of evidence attribution rather than a complex evaluation of logical validity, ensuring the task is computationally simpler and more robust than the reasoning process itself. The verifier is deployed on a dedicated server utilizing 2 GPUs with a temperature of 0.0. We present the specific prompt templates in Appendix~\ref{app:prompt_template}.

\paragraph{Evidence for OOD Benchmarks.} To evaluate the out-of-distribution (OOD) generalization of our approach, we conduct inference-only evaluations on GSM8k and MATH500. For these benchmarks, we utilize the ground-truth solution steps as the evidence set solely to verify the generated reasoning chains during the testing phase. This setup ensures a rigorous assessment of how well our verification mechanism, trained on multi-hop QA, generalizes to unseen mathematical reasoning tasks in a zero-shot manner. The specific prompts used for verification are detailed in Appendix~\ref{app:prompt_template}.

\section{Extended Experimental Results}
\label{app:more_results}

\subsection{Additional Results for Main Experiments}
\label{subsec:detailed_main_results}
We present a more granular analysis of Multi-hop QA performance, specifically examining the model's behavior on answerable versus unanswerable subsets of the test data in Table~\ref{tab:detailed_main}.
\begin{table*}[!t]
\centering
\caption{Fine-grained results comparing \textbf{FaithRL} against strong baselines across three multi-hop QA benchmarks. We report \textbf{Accuracy (Acc)} and \textbf{Truthful Helpfulness Score (THS)} where higher is better ($\uparrow$), and \textbf{Error Rate (Err)} where lower is better ($\downarrow$). The performance is reported separately for Answerable (Ans.) and Unanswerable (Unans.) subsets. The best performance is highlighted in \textbf{bold}, and the second best is \underline{underlined}.}
\label{tab:detailed_main}
\vspace{-0.5em}

\scriptsize 
\setlength{\tabcolsep}{1pt} 

\begin{tabularx}{\textwidth}{
    l 
    @{\hskip 4pt} 
    YYYYYY 
    @{\hskip 4pt} 
    YYYYYY 
    @{\hskip 4pt} 
    YYYYYY 
    @{\hskip 4pt} 
    YYY    
}
\toprule
& \multicolumn{6}{c}{\textbf{2WikiMultihopQA-Full}} 
& \multicolumn{6}{c}{\textbf{HotpotQA-Full}} 
& \multicolumn{6}{c}{\textbf{MuSiQue-Full}} 
& \multicolumn{3}{c}{\textbf{Average}} \\ 

\cmidrule(lr){2-7} \cmidrule(lr){8-13} \cmidrule(lr){14-19} \cmidrule(lr){20-22}

& \multicolumn{3}{c}{Ans.} & \multicolumn{3}{c}{Unans.} 
& \multicolumn{3}{c}{Ans.} & \multicolumn{3}{c}{Unans.} 
& \multicolumn{3}{c}{Ans.} & \multicolumn{3}{c}{Unans.} 
& \multirow{2}{*}{\textbf{A}} & \multirow{2}{*}{\textbf{H}} & \multirow{2}{*}{\textbf{T}} \\ 

\cmidrule(lr){2-4} \cmidrule(lr){5-7} 
\cmidrule(lr){8-10} \cmidrule(lr){11-13} 
\cmidrule(lr){14-16} \cmidrule(lr){17-19} 

{\bf Method} 
& \textbf{A} & \textbf{H} & \textbf{T} & \textbf{A} & \textbf{H} & \textbf{T} 
& \textbf{A} & \textbf{H} & \textbf{T} & \textbf{A} & \textbf{H} & \textbf{T} 
& \textbf{A} & \textbf{H} & \textbf{T} & \textbf{A} & \textbf{H} & \textbf{T} 
& & & \\ 

\midrule

\multicolumn{22}{l}{\textit{\textbf{Llama3.1-8B-Instruct}}} \\
\multicolumn{22}{l}{\textsc{Prompting Baseline}} \\
\quad Prompting & 56.6 & 30.4 & 0.0 & 81.7 & 18.3 & 0.0 & 70.5 & 24.6 & 0.0 & 56.8 & 43.2 & 0.0 & 41.6 & 32.3 & 0.0 & 66.6 & 33.4 & 0.0 & 62.3 & 30.4 & 0.0 \\
\multicolumn{22}{l}{\textsc{Refusal-Aware SFT}} \\
\quad R-Tuning  & 18.7 & 20.5 & -19.4 & \textbf{95.1} & \underline{4.9} & \textbf{73.2} & 34.2 & 15.8 & -11.1 & \textbf{94.1} & \textbf{5.9} & \textbf{86.3} & 35.5 & 21.0 & 8.5 & \textbf{96.3} & \textbf{3.7} & \textbf{88.9} & 62.3 & 12.0 & 37.7 \\
\quad CRaFT     & 9.7 & 18.2 & -24.2 & 88.5 & 11.5 & 37.2 & 34.5 & 11.4 & 1.8 & \underline{91.5} & \underline{8.5} & \underline{80.3} & 22.2 & \textbf{8.4} & 11.4 & \underline{96.1} & \underline{3.9} & \underline{88.3} & 57.1 & \underline{10.4} & 35.8 \\
\multicolumn{22}{l}{\textsc{Confidence-Based Abstention}} \\
\quad CTSF      & 37.4 & 19.1 & 1.8 & 89.5 & 10.5 & 42.6 & 43.4 & 11.4 & 10.7 & 89.8 & 10.2 & 76.4 & 20.6 & 10.9 & 6.6 & 93.0 & 7.0 & 79.0 & 62.3 & 11.5 & 38.7 \\
\quad RLCR & 64.4 & 28.3 & 11.7 & 56.7 & 43.3 & -136.6 & 72.3 & 26.3 & -3.1 & 46.3 & 53.7 & -24.3 & 45.1 & 39.0 & -5.1 & 56.1 & 43.9 & -31.4 & 56.8 & 39.1 & -23.3 \\
\quad GRPO + CTSF & 67.6 & 27.0 & 17.3 & 33.6 & 66.4 & -262.8 & 77.0 & 15.5 & 32.5 & 58.9 & 41.1 & 4.9 & 61.9 & 17.1 & 39.9 & 55.7 & 44.3 & -32.6 & 59.1 & 35.2 & -13.0 \\
\multicolumn{22}{l}{\textsc{RLVR Methods}} \\
\quad GRPO      & \textbf{83.5} & 13.3 & \underline{58.7} & 88.8 & 11.2 & 38.8 & \textbf{88.1} & \underline{9.8} & \underline{60.0} & 64.1 & 35.9 & 16.9 & \underline{68.9} & 16.7 & 47.4 & 90.5 & 9.5 & 71.6 & \underline{80.7} & 16.1 & 47.7 \\
\quad TruthRL   & 81.8 & 12.0 & 59.5 & 88.8 & 11.2 & 38.8 & 84.5 & \textbf{7.5} & \textbf{63.0} & 71.6 & 28.4 & 34.3 & \textbf{70.0} & \underline{9.5} & \textbf{57.8} & 90.2 & 9.8 & 70.7 & 81.2 & 13.1 & \underline{54.4}\\
\multicolumn{22}{l}{\textsc{Factuality-Driven RL}} \\
\quad FSPO   & 62.8 & \textbf{10.8} & 42.7 & 92.5 & 7.5 & 59.0 & 73.1 & 11.3 & 40.7 & 83.4 & 16.6 & 61.6 & 36.1 & 15.1 & 16.7 & 90.3 & 9.7 & 71.0 & 73.0 & 11.8 & 48.8 \\
\quad KnowRL   & \underline{82.6} & 13.8 & 56.9 & 74.9 & 25.1 & -37.2 & \underline{87.1} & 10.0 & 58.4 & 76.7 & 23.3 & 46.1 & 67.4 & 17.0 & 45.5 & 83.2 & 16.8 & 49.7 & 78.7 & 17.7 & 42.4 \\
\rowcolor{gray!15} \textbf{FaithRL} (Ours) & 81.7 & \underline{11.6} & \textbf{60.1} & \underline{93.3} & \textbf{6.7} & \underline{63.4} & 83.9 & \underline{9.8} & 55.8 & 87.0 & 13.0 & 69.9 & 65.3 & 12.4 & \underline{49.3} & 94.7 & 5.3 & 84.1 & \textbf{84.3} & \textbf{9.8} & \textbf{64.2} \\

\midrule
\multicolumn{22}{l}{\textit{\textbf{Qwen2.5-7B-Instruct}}} \\
\multicolumn{22}{l}{\textsc{Prompting Baseline}} \\
\quad Prompting & 50.4 & 18.8 & 0.0 & \underline{92.6} & \underline{7.4} & 0.0 & 66.2 & 17.7 & 0.0 & 79.2 & 20.8 & 0.0 & 31.2 & 19.7 & 0.0 & 87.2 & 12.8 & 0.0 & 67.8 & 16.2 & 0.0 \\
\multicolumn{22}{l}{\textsc{Refusal-Aware SFT}} \\
\quad R-Tuning  &26.6 & 35.4 & -68.3 & 91.7 & 8.3 & -12.2 & 36.8 & 23.2 & -50.0 & \textbf{92.8} & \textbf{7.2} & \textbf{65.4} & 34.9 & 23.6 & -2.5 & 96.1 & 3.9 & 69.5 & 63.2 & 16.9 & -7.5 \\
\quad CRaFT     & 9.8 & 14.5 & -29.1 & \textbf{92.7} & \textbf{7.3} & \textbf{1.4} & 38.0 & 11.5 & -5.0 & 90.9 & 9.1 & 56.3 & 19.8 & \underline{7.2} & 8.4 & \underline{96.6} & \underline{3.4} & \underline{73.4} & 58.0 & \underline{8.8} & 21.2 \\
\multicolumn{22}{l}{\textsc{Confidence-Based Abstention}} \\
\quad CTSF      & 43.8 & 15.7 & 1.71 & 67.5 & 32.5 & -339.2 & 39.3 & 10.0 & 1.89 & 87.2 & 12.8 & 38.5 & 15.4 & \textbf{7.1} & 4.16 & 92.1 & 7.9 & 38.3 & 57.6 & 14.3 & -2.2 \\
\quad RLCR & 60.5 & 23.2 & -1.7 & 69.3 & 30.7 & -314.9 & 71.9 & 21.6 & -8.9 & 62.3 & 37.7 & -81.3 & 33.6 & 28.5 & -11.5 & 70.4 & 29.6 & -131.3 & 61.3 & 28.5 & -58.0 \\
\quad GRPO + CTSF & 71.0 & 16.9 & 25.7 & 58.5 & 41.5 & -460.8 & 71.9 & 10.7 & 31.9 & 80.2 & 19.8 & 4.81 & 42.6 & 10.7 & 25.7 & 84.3 & 15.7 & -22.7 & 68.1 & 19.2 & -12.3 \\
\multicolumn{22}{l}{\textsc{RLVR Methods}} \\
\quad GRPO      & \textbf{82.0} & 14.1 & 44.2 & 87.9 & 12.1 & -63.5 & \textbf{86.8} & 11.1 & 45.3 & 81.6 & 18.4 & 11.5 & \textbf{65.9} & 23.5 & 28.7 & 92.9 & 7.1 & 44.5 & \textbf{82.8} & 14.4 & 22.5 \\
\quad TruthRL   & \underline{78.1} & \underline{8.9} & \underline{54.2} & 89.4 & 10.6 & -43.2 & \underline{82.1} & \underline{8.4} & \textbf{50.7} & 88.1 & 11.9 & 42.8 & \underline{59.6} & 11.7 & \textbf{41.1} & 95.7 & 4.3 & 66.4 & 82.2 & 9.3 & \underline{43.3} \\
\multicolumn{22}{l}{\textsc{Factuality-Driven RL}} \\
\quad FSPO   & 56.4 & 18.1 & 7.9 & 88.3 & 11.7 & -58.1 & 70.4 & 16.2 & 9.8 & 64.0 & 36.0 & -73.1 & 31.1 & 13.7 & 9.4 & 86.0 & 14.0 & -9.4 & 66.0 & 18.3 & -10.6 \\
\quad KnowRL   & 70.4 & 24.8 & 3.9 & 77.1 & 22.9 & -209.5 & 82.0 & 15.1 & 25.5 & 77.4 & 22.6 & -8.7 & 56.5 & 26.3 & 14.8 & 91.1 & 8.9 & 30.5 & 75.8 & 20.1 & -8.3\\
\rowcolor{gray!15} \textbf{FaithRL} (Ours) & 78.0 & \textbf{8.6} & \textbf{54.9} & \textbf{92.7} & \textbf{7.3} & \textbf{1.4} & 79.8 & \textbf{8.3} & \underline{48.8} & \underline{92.4} & \underline{7.6} & \underline{63.5} & 54.2 & 9.1 & \underline{39.8} & \textbf{97.2} & \textbf{2.8} & \textbf{78.1} & \underline{82.4} & \textbf{7.3} & \textbf{51.8} \\
\bottomrule
\end{tabularx}
\vspace{-1em}
\end{table*}

\subsection{Extended Analysis of Training Dynamics}
\label{subsec:detailed_training}
We visualize the training dynamics of step-wise faithfulness in Figure~\ref{fig:all_ratios}. We track the evolution of reasoning steps for TruthRL, GRPO, and FaithRL, dividing the metrics into four categories based on the sample polarity (Positive vs. Negative) and step validity (Faithful vs. Unfaithful). As illustrated, FaithRL demonstrates superior reasoning alignment capabilities: it achieves the highest proportion of faithful steps on positive samples (Figure~\ref{fig:all_ratios}a) while successfully suppressing unfaithful reasoning steps on negative samples to the lowest level among all baselines (Figure~\ref{fig:all_ratios}d).

\begin{figure*}[t] 
    \centering
    \begin{subfigure}[b]{0.48\textwidth}
        \centering
        \includegraphics[width=\linewidth]{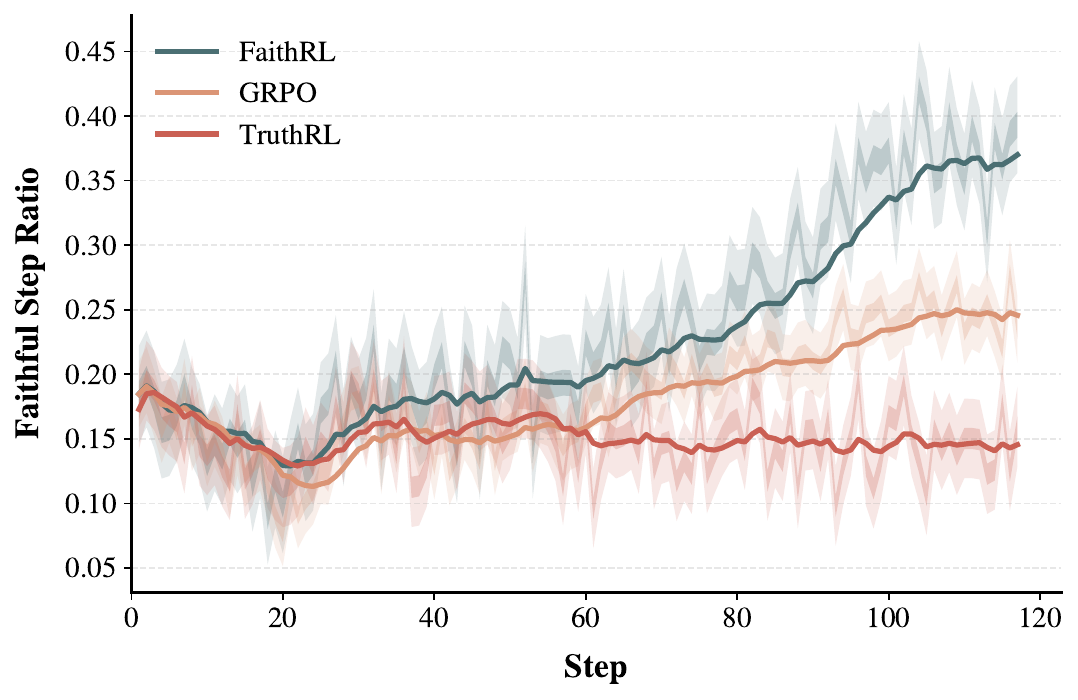}
        \caption{Faithful Step Ratio (Positive)}
        \label{fig:faithful_pos}
    \end{subfigure}
    \hfill 
    \begin{subfigure}[b]{0.48\textwidth}
        \centering
        \includegraphics[width=\linewidth]{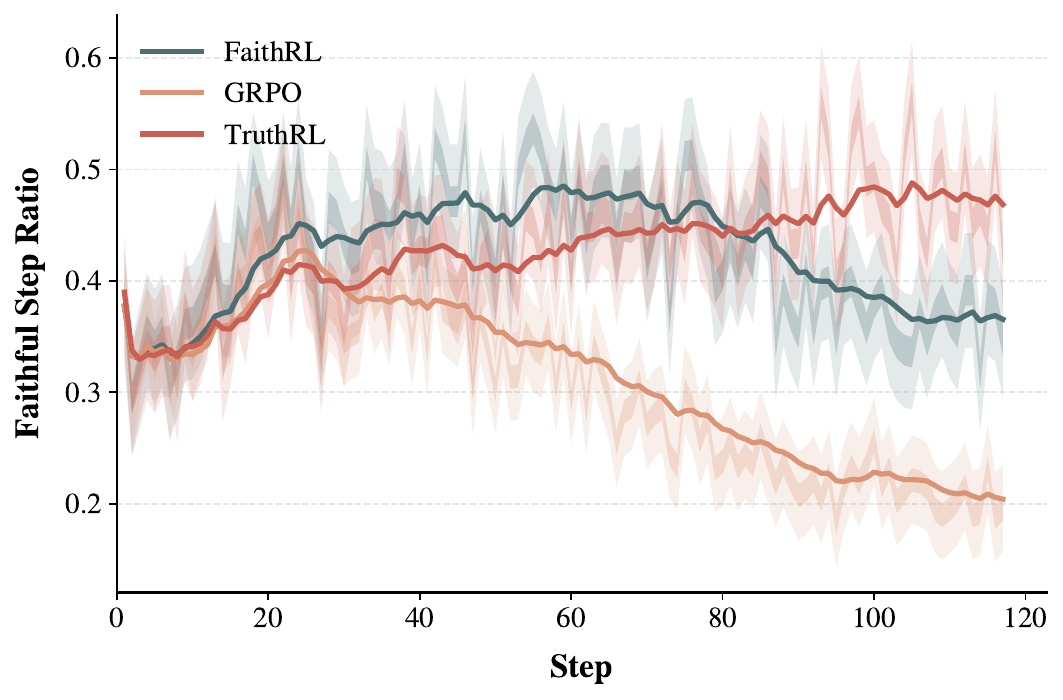}
        \caption{Unfaithful Step Ratio (Positive)}
        \label{fig:unfaithful_pos}
    \end{subfigure}
    

    \begin{subfigure}[b]{0.48\textwidth}
        \centering
        \includegraphics[width=\linewidth]{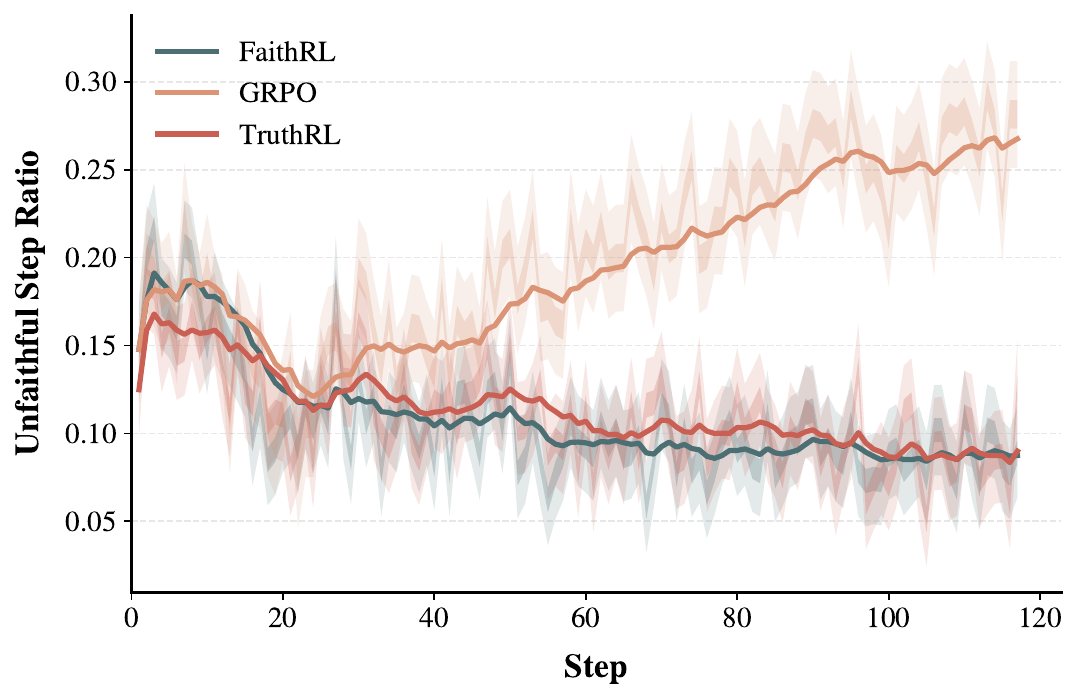}
        \caption{Faithful Step Ratio (Negative)}
        \label{fig:faithful_neg}
    \end{subfigure}
    \hfill
    \begin{subfigure}[b]{0.48\textwidth}
        \centering
        \includegraphics[width=\linewidth]{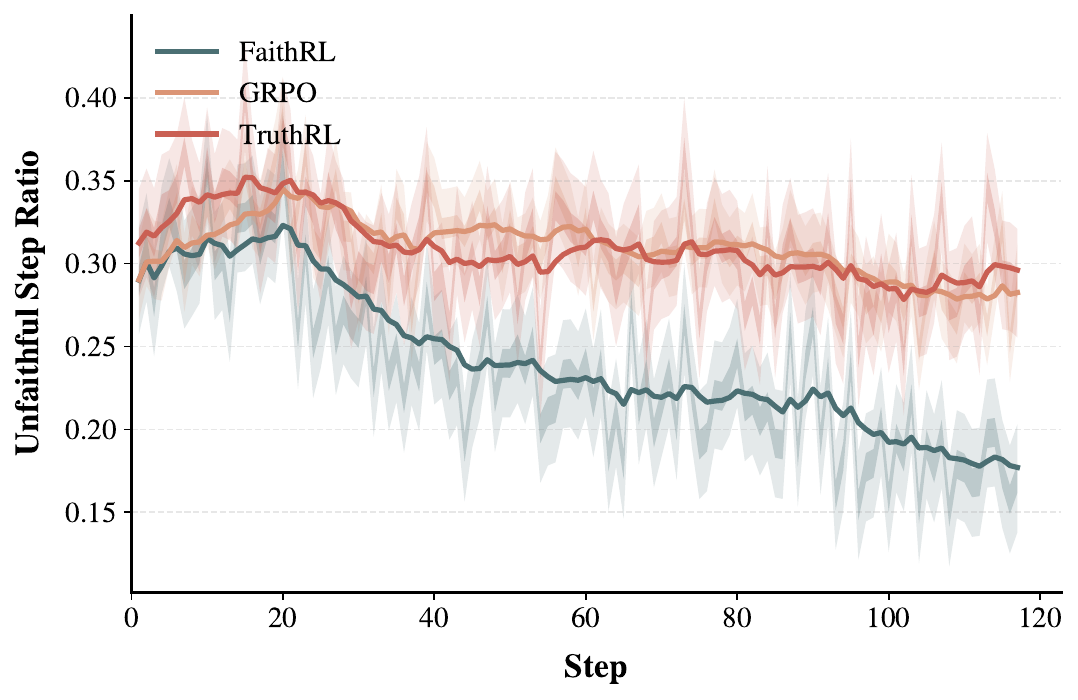}
        \caption{Unfaithful Step Ratio (Negative)}
        \label{fig:unfaithful_neg}
    \end{subfigure}
    
    \caption{Analysis of Faithful and Unfaithful Step Ratios on positive and negative samples. TruthRL, GRPO, and FaithRL are compared across different steps.}
    \label{fig:all_ratios}
\end{figure*}

\subsection{Extended Results on OOD Generalization}
\label{subsec:detailed_ood}
We also provide a comprehensive breakdown of the OOD generalization experiments on GSM8k and MATH500 datasets. Table \ref{tab:detailed_ood} details the performance metrics and step-wise reasoning validity for both FaithRL and the GRPO baseline.
\begin{table}[h]
    \centering
    \caption{Detailed Performance and Faithfulness Analysis on OOD Datasets. The table presents outcome metrics (Correct Rate, Miss Rate, Hallucination Rate) and process-level metrics (Step Validity across different outcome types). Values in parentheses indicate the raw counts (numerator/denominator).}
    \label{tab:detailed_ood}
    \resizebox{\textwidth}{!}{
    \begin{tabular}{llccccccc}
        \toprule
        & & \multicolumn{3}{c}{\textbf{Outcome Metrics}} & \multicolumn{4}{c}{\textbf{Process Validity}} \\
        \cmidrule(lr){3-5} \cmidrule(lr){6-9}
        \textbf{Dataset} & \textbf{Method} & 
        \textbf{Correctness} & \textbf{Miss} & \textbf{Hallucination} & 
        \textbf{Faith. (Correct)} & \textbf{Faith. (Miss)} & \textbf{Faith. (Hallu.)} & 
        \textbf{Overall Faith.} \\ 
        \midrule
        \multirow{4}{*}{\textbf{GSM8k}} 
        & \textbf{FaithRL} & \textbf{86.9\%} & 0.9\% & \textbf{12.2\%} & \textbf{79.5\%} & 21.7\% & 42.6\% & \textbf{74.1\%} \\
        & & (1146/1319) & (12/1319) & (161/1319) & (3303/4156) & (10/46) & (271/636) & (3584/4838) \\
        \cmidrule{2-9}
        & GRPO & 72.3\% & 1.2\% & 26.5\% & 72.7\% & 35.3\% & 55.4\% & 67.5\% \\
        & & (954/1319) & (16/1319) & (349/1319) & (2994/4120) & (30/85) & (843/1523) & (3867/5728) \\
        \midrule
        \multirow{4}{*}{\textbf{MATH500}} 
        & \textbf{FaithRL} & \textbf{53.6\%} & 1.6\% & \textbf{44.8\%} & \textbf{64.2\%} & 19.5\% & 41.0\% & \textbf{52.3\%} \\
        & & (268/500) & (8/500) & (224/500) & (683/1064) & (8/41) & (407/993) & (1098/2098) \\
        \cmidrule{2-9}
        & GRPO & 46.6\% & 6.2\% & 47.2\% & 55.0\% & 17.9\% & 41.0\% & 45.2\% \\
        & & (233/500) & (31/500) & (236/500) & (765/1391) & (48/268) & (628/1532) & (1441/3191) \\
        \bottomrule
    \end{tabular}
    }
\end{table}

\subsection{Extended Results on Hyperparameter Sensitivity}
\label{subsec:detailed_ablation_alpha}
We also provide a comprehensive breakdown of the hyperparameter sensitivity analysis regarding $\alpha$ on the Qwen-2.5-7B-Instruct model. Table \ref{tab:detailed_ablation_alpha} details the performance metrics across all three multi-hop QA benchmarks as $\alpha$ varies from 0 to 1, illustrating the impact of the trade-off between standard rewards and faithfulness constraints.

\begin{table*}[!t]
\centering
\caption{Hyperparameter sensitivity analysis of $\alpha$ on Qwen2.5-7B-Instruct. $\alpha$ balances the trade-off between standard reinforcement learning rewards and our proposed faithfulness constraints. Note that $\alpha=1.0$ is equivalent to the vanilla RL method with our geometric reward. The chosen setting for our main experiments is highlighted in gray.}
\label{tab:detailed_ablation_alpha}
\vspace{-0.5em}

\scriptsize 
\setlength{\tabcolsep}{1pt}

\begin{tabularx}{\textwidth}{
    l 
    @{\hskip 4pt} 
    YYYY 
    @{\hskip 4pt} 
    YYYY 
    @{\hskip 4pt} 
    YYYY 
    @{\hskip 4pt} 
    YY    
}
\toprule
& \multicolumn{4}{c}{\textbf{2WikiMultihopQA-Full}} 
& \multicolumn{4}{c}{\textbf{HotpotQA-Full}} 
& \multicolumn{4}{c}{\textbf{MuSiQue-Full}} 
& \multicolumn{2}{c}{\textbf{Average}} \\ 

\cmidrule(lr){2-5} \cmidrule(lr){6-9} \cmidrule(lr){10-13} \cmidrule(lr){14-15}

& \multicolumn{2}{c}{Ans.} & \multicolumn{2}{c}{Unans.} 
& \multicolumn{2}{c}{Ans.} & \multicolumn{2}{c}{Unans.} 
& \multicolumn{2}{c}{Ans.} & \multicolumn{2}{c}{Unans.} 
& \multirow{2}{*}{\textbf{A}} & \multirow{2}{*}{\textbf{H}} \\ 

\cmidrule(lr){2-3} \cmidrule(lr){4-5} 
\cmidrule(lr){6-7} \cmidrule(lr){8-9} 
\cmidrule(lr){10-11} \cmidrule(lr){12-13} 

{\bf Setting} 
& \textbf{A} & \textbf{H} & \textbf{A} & \textbf{H} 
& \textbf{A} & \textbf{H} & \textbf{A} & \textbf{H} 
& \textbf{A} & \textbf{H} & \textbf{A} & \textbf{H} 
& & \\ 

\midrule
Prompting & 50.4 & 18.8 & 92.6 & 7.4 & 66.2 & 17.7 & 79.2 & 20.8 & 31.2 & 19.7 & 87.2 & 12.8 & 67.8 & 16.2 \\ 
\midrule
\rowcolor{gray!15} $\alpha=0$ (Ours) & 78.0 & 8.6 & 92.7 & 7.3 & 79.8 & 8.3 & 92.4 & 7.6 & 54.2 & 9.1 & 97.2 & 2.8 & 82.4 & 7.3 \\

$\alpha=0.25$ & 76.8 & 8.9 & 90.5 & 9.5 & 81.6 & 8.1 & 90.3 & 9.7 & 51.0 & 9.5 & 96.7 & 3.3 & 81.1 & 8.2 \\

$\alpha=0.5$ & 74.5 & 7.7 & 94.3 & 5.7 & 78.2 & 7.8 & 89.6 & 10.4 & 50.0 & 9.3 & 95.8 & 4.2 & 80.4 & 7.5 \\

$\alpha=0.75$ & 74.2 & 8.0 & 93.3 & 6.7 & 80.5 & 7.5 & 88.5 & 11.5 & 44.0 & 9.6 & 96.4 & 3.6 & 79.5 & 7.8 \\

$\alpha=1.0$ (Vanilla) & 73.2 & 6.8 & 93.2 & 6.8 & 82.0 & 7.5 & 85.6 & 14.4 & 50.7 & 12.2 & 94.5 & 5.5 & 79.8 & 8.8 \\

\bottomrule
\end{tabularx}
\vspace{-1em}
\end{table*}

\subsection{Extended Results on Component Analysis}
\label{subsec:detailed_ablation_component}
We utilize the Qwen-2.5-7B-Instruct model as the backbone for these ablation studies. Table~\ref{tab:detailed_ablation_component} details the performance metrics across all three multi-hop QA benchmarks, illustrating the contribution of each module to the overall performance.

\begin{table*}[!t]
\centering
\caption{Detailed ablation results on component effectiveness.}
\label{tab:detailed_ablation_component}
\vspace{-0.5em}

\scriptsize 
\setlength{\tabcolsep}{1pt}

\begin{tabularx}{\textwidth}{
    l 
    @{\hskip 4pt} 
    YYYY 
    @{\hskip 4pt} 
    YYYY 
    @{\hskip 4pt} 
    YYYY 
    @{\hskip 4pt} 
    YY    
}
\toprule
& \multicolumn{4}{c}{\textbf{2WikiMultihopQA-Full}} 
& \multicolumn{4}{c}{\textbf{HotpotQA-Full}} 
& \multicolumn{4}{c}{\textbf{MuSiQue-Full}} 
& \multicolumn{2}{c}{\textbf{Average}} \\ 

\cmidrule(lr){2-5} \cmidrule(lr){6-9} \cmidrule(lr){10-13} \cmidrule(lr){14-15}

& \multicolumn{2}{c}{Ans.} & \multicolumn{2}{c}{Unans.} 
& \multicolumn{2}{c}{Ans.} & \multicolumn{2}{c}{Unans.} 
& \multicolumn{2}{c}{Ans.} & \multicolumn{2}{c}{Unans.} 
& \multirow{2}{*}{\textbf{A}} & \multirow{2}{*}{\textbf{H}} \\ 

\cmidrule(lr){2-3} \cmidrule(lr){4-5} 
\cmidrule(lr){6-7} \cmidrule(lr){8-9} 
\cmidrule(lr){10-11} \cmidrule(lr){12-13} 

{\bf Setting} 
& \textbf{A} & \textbf{H} & \textbf{A} & \textbf{H} 
& \textbf{A} & \textbf{H} & \textbf{A} & \textbf{H} 
& \textbf{A} & \textbf{H} & \textbf{A} & \textbf{H} 
& & \\ 

\midrule
Prompting & 50.4 & 18.8 & 92.6 & 7.4 & 66.2 & 17.7 & 79.2 & 20.8 & 31.2 & 19.7 & 87.2 & 12.8 & 67.8 & 16.2 \\ 
\midrule

GRPO & 82.0 & 14.1 & 87.9 & 12.1 & 86.8 & 11.1 & 81.6 & 18.4 & 65.9 & 23.5 & 92.9 & 7.1 & 82.8 & 14.4 \\

+$\mathcal{R}_{geo}$ & 73.2 & 6.8 & 93.2 & 6.8 & 82.0 & 7.5 & 85.6 & 14.4 & 50.7 & 12.2 & 94.5 & 5.5 & 79.8 & 8.8 \\

+FAAM & 78.2 & 8.5 & 88.1 & 11.9 & 83.2 & 7.1 & 73.1 & 26.9 & 52.4 & 9.5 & 91.7 & 8.3 & 77.8 & 12.0 \\

\rowcolor{gray!15} FaithRL & 78.0 & 8.6 & 92.7 & 7.3 & 79.8 & 8.3 & 92.4 & 7.6 & 54.2 & 9.1 & 97.2 & 2.8 & 82.4 & 7.3 \\

\bottomrule
\end{tabularx}
\vspace{-1em}
\end{table*}

\section{Experimental Computational Cost}
\label{app:comp_cost}

\paragraph{Experimental Setup Details.}
To ensure a rigorous and fair evaluation, we tracked resource consumption across the full training pipeline.
For \textbf{FLOPs}, we recorded the `flops/total`, representing the cumulative floating-point operations since the start of training. This metric encompasses all major stages: actor/critic updates, log-probability/value forward passes, vLLM rollout generation, and server-side LLM-judge computations (estimated via API usage tokens).
For \textbf{GPU Hours}, we calculated the total usage of the 4 NVIDIA H200 GPUs dedicated to model training and gradient updates. Additionally, to account for the overhead of the external reward model, we included the consumption of the 2 H200 GPUs hosting the LLM server, scaled by their average Streaming Multiprocessor (SM) Utilization. This approach ensures that the reported cost reflects the actual hardware burden rather than just wall-clock time.

\paragraph{Extended Results on GPU Hours.}
We provide a detailed breakdown of the computational cost comparison between GRPO and FaithRL. The total GPU hours are composed of two parts: the time spent on model parameter updates (using 4 GPUs) and the equivalent time cost for the LLM Judge server.

Table~\ref{tab:gpu_hours_breakdown} presents the detailed calculation for the total GPU hours consumed by both methods.

\begin{table}[htbp]
    \centering
    \caption{Breakdown of total GPU hours calculation.}
    \label{tab:gpu_hours_breakdown}
    \begin{tabular}{lcccc}
        \toprule
        \textbf{Method} & \textbf{Update Time (h)} & \textbf{Judge Server Cost (h)} & \textbf{Total GPU Hours} & \textbf{Increase} \\
        \midrule
        GRPO & $4 \times 22.4 = 89.6$ & $22.4 \times 22.37\% \times 2 \approx 10.0$ & 99.6 & - \\
        FaithRL & $4 \times 24.8 = 99.2$ & $24.8 \times 28.17\% \times 2 \approx 14.0$ & 113.2 & +13.7\% \\
        \bottomrule
    \end{tabular}
\end{table}

The "Update Time" refers to the wall-clock time for training on 4 GPUs. The "Judge Server Cost" is calculated by converting the active judge server usage (percentage of training time $\times$ number of judge GPUs, here assumed to be 2) into equivalent GPU hours. FaithRL incurs slightly higher costs due to the additional verification steps, but the overall increase remains modest.

\paragraph{Theoretical analysis of efficiency.}
To explain why FaithRL maintains high efficiency despite increased verification frequency, we formalize the total training time per iteration $T_{total}$. We define the following notation: $G$ denotes the number of rollouts per prompt; $N$ represents the average number of reasoning steps (empirically $N \approx 3.5$); $T_{gen}$ and $T_{update}$ represent the time consumed by the policy model for generation and parameter updates, respectively; and $T_{req}$ denotes the latency for interactions with the Judge Server.

The cost models for the baseline and our method are formulated as follows:
\begin{itemize}
    \item \textbf{GRPO (baseline):} Performs a single outcome verification per sample.
    \[ T_{GRPO} \approx T_{gen} + G \times T_{req}^{outcome} + T_{update} \]
    \item \textbf{FaithRL (ours):} Performs one outcome verification plus $N$ step verifications per sample.
    \[ T_{FaithRL} \approx T_{gen} + G \times (T_{req}^{outcome} + T_{overhead}^{steps}) + T_{update} \]
\end{itemize}

While the verification term appears to scale by $N$, two critical factors mitigate the actual wall-clock impact, keeping the overhead below 15\%:

\begin{enumerate}
    \item \textbf{KV-cache reuse (context caching):} This is the dominant factor. In GRPO, the outcome judge processes the full context (document + question + reasoning chain), incurring a costly prefill (context encoding) operation. In FaithRL, the outcome judge and step judges share the identical document context. Once the server (e.g., vLLM) processes the initial request, the Key-Value (KV) cache for the context is resident in VRAM. Consequently, subsequent step verifications primarily involve the much faster decoding phase, bypassing the expensive prefill. Thus, the marginal cost of a step verification is significantly lower than an outcome verification ($T_{req}^{step} < T_{req}^{outcome}$).
    
    \item \textbf{Server concurrency and utilization:} Our Judge Server is configured with a high concurrency limit (\texttt{MAX\_NUM\_SEQS=64}). Empirical monitoring reveals that the GPU utilization on the server side remains relatively low (approximately 20\%--30\%) for both GRPO and FaithRL. This indicates that the server is not compute-bound and has ample capacity. Therefore, increasing the request frequency from 1 (outcome only) to $1+N$ (outcome + steps) does not introduce a linear latency penalty. The server can efficiently process these additional, lightweight step-verification requests concurrently without significant queuing delay, resulting in a sub-linear increase in total verification time ($T_{overhead}^{steps} < N \times T_{req}^{step\_isolated}$).
\end{enumerate}

In conclusion, due to efficient caching and the server's capacity to handle increased concurrency without saturation, the additional computational overhead of FaithRL remains marginal.

\section{Case Study}
\label{app:case_study}

To provide concrete insights into the behavioral differences between FaithRL and GRPO, we present two representative cases from the HotpotQA-Full dataset. These examples illustrate how our method enhances both reasoning capability and faithfulness compared to outcome-driven baselines.

\subsection{Analysis of Answerable Scenarios}
In the answerable case (Table \ref{tab:case_answerable}), the model is required to perform multi-hop reasoning: identifying a specific compilation album, finding a track on it, and verifying its sampling source.
\begin{itemize}
    \item \textbf{FaithRL (Ours):} Our model successfully synthesizes information from multiple documents. It correctly identifies that ``Fantastic Voyage: The Greatest Hits'' contains tracks from ``Gangsta's Paradise'', and links this to the sampling information of Stevie Wonder. The reasoning chain is logically sound and leads to the correct answer.
    \item \textbf{GRPO (Baseline):} The baseline model fails to make the necessary logical connection. Although it retrieves relevant information about ``Gangsta's Paradise'', it fails to link the compilation album context correctly, leading to an unnecessary refusal (``I don't know''). This highlights the issue of over-conservativeness or reasoning failure in standard RL approaches when faced with complex evidence integration.
\end{itemize}

\subsection{Analysis of Unanswerable Scenarios}
In the unanswerable case (Table \ref{tab:case_unanswerable}), the model must identify a specific location of an uprising. While the general context (Haitian Revolution) is present, the specific location ``Saint-Domingue'' is missing from the provided evidence.
\begin{itemize}
    \item \textbf{FaithRL (Ours):} Our model demonstrates superior faithfulness. It correctly identifies the event and the general context (Haiti) but crucialy notes that ``the documents do not provide a specific location within Haiti''. Consequently, it refuses to answer, adhering strictly to the provided knowledge.
    \item \textbf{GRPO (Baseline):} The baseline model succumbs to hallucination or internal knowledge bias. It attempts to answer with ``Haiti'', which is technically the resulting state but not the specific answer required by the ground truth context (which was pruned). This illustrates the tendency of standard RL to ``guess'' based on partial associations, violating the faithfulness constraint.
\end{itemize}

\subsection{Reasoning Faithfulness and Step-wise Contribution}
Beyond correctness, we observe a significant disparity in reasoning efficiency. As demonstrated in the examples, FaithRL typically resolves complex multi-hop queries within approximately 3 concise steps. In contrast, the GRPO-trained model often requires more than double the number of steps to arrive at a conclusion. Crucially, regarding reasoning faithfulness, every step generated by FaithRL validly contributes to the final derivation, achieving a 100\% valid step ratio in these cases. Conversely, the GRPO baseline exhibits a high proportion of redundant or irrelevant reasoning steps, diluting the logical density and indicating a reliance on search-based trial-and-error rather than systematic deduction.

\begin{table}[h]
    \centering
    \small
    \begin{tabular}{p{0.95\textwidth}}
        \hline
        \textbf{Question:} What song by the American rapper Coolio is on his greatest hits album ``Fantiastic Voyage'' and samples from a Stevie Wonder song? \\
        \hline
        \textbf{FaithRL (Ours) Output:} \\
        1. First, identify the greatest hits album by Coolio that is mentioned. From Document 'Fantastic Voyage: The Greatest Hits', we learn "Fantastic Voyage: The Greatest Hits is a compilation album by rapper Coolio, released in 2001. It includes the track 'Aw Here it Goes' he contributed as a main theme to TV series 'Kenan \& Kel'. Other songs originally appear on his first three albums, 'It Takes a Thief', 'Gangsta's Paradise' and 'My Soul'. " This sentence mentions that the album is "Fantastic Voyage" and that it includes tracks from Coolio's other albums. However, it does not directly state which song samples a Stevie Wonder song. \\
        2. Next, identify a song by Coolio that samples a Stevie Wonder song. From Document 'Gangsta's Paradise', we learn "'Gangsta's Paradise" is a song by American rapper Coolio, featuring singer L.V. The song was released on Coolio's album of the same name, as well as the soundtrack for the 1995 film "Dangerous Minds". It samples the chorus and instrumentation of Stevie Wonder's 1976 song "Pastime Paradise". " This sentence states that "Gangsta's Paradise" by Coolio samples a Stevie Wonder song. Therefore, "Gangsta's Paradise" is the song by Coolio that samples a Stevie Wonder song. \\
        3. Verify if "Gangsta's Paradise" is on the greatest hits album "Fantastic Voyage". From Document 'Fantastic Voyage: The Greatest Hits', we learn "Fantastic Voyage: The Greatest Hits is a compilation album by rapper Coolio, released in 2001. It includes the track 'Aw Here it Goes' he contributed as a main theme to TV series 'Kenan \& Kel'. Other songs originally appear on his first three albums, 'It Takes a Thief', 'Gangsta's Paradise' and 'My Soul'. " This sentence mentions that "Gangsta's Paradise" is one of the songs originally appearing on his first three albums, implying it is on the greatest hits album "Fantastic Voyage".\\
        \textbf{Answer:} Gangsta's Paradise \quad \textbf{(Reward: $y_0$)} \\
        \hline
        \textbf{GRPO (Baseline) Output:} \\
        1. First, we need to identify a song by Coolio that is on his greatest hits album "Fantastic Voyage" and samples from a Stevie Wonder song. From the document "Fantastic Voyage: The Greatest Hits", we learn that this album by Coolio includes the track "Aw Here it Goes" which he contributed as a main theme to the TV series "Kenan \& Kel". \\ 
        2. Next, we need to find out if this song or any other song by Coolio samples from a Stevie Wonder song. The document "Gangsta's Paradise" states that "Gangsta's Paradise" by Coolio samples the chorus and instrumentation of Stevie Wonder's 1976 song "Pastime Paradise". \\ 
        3. Given that "Gangsta's Paradise" is a Coolio song that samples a Stevie Wonder song, and it is also mentioned that "Gangsta's Paradise" is featured on the album of the same name, which is not "Fantastic Voyage", we need to find another Coolio song on "Fantastic Voyage" that samples a Stevie Wonder song. \\ 
        4. The document "Fantastic Voyage (Coolio song)" tells us that "Fantastic Voyage" by Coolio samples "Fantastic Voyage" by Lakeside, and it is on his debut album "It Takes a Thief". \\
        5. The document "Curtain Falls" mentions that the background music of "Curtain Falls" is almost identical to that of Coolio's 1995 hit, "Gangsta's Paradise", as they both share usage of a sample of Stevie Wonder's "Pastime Paradise". \\
        6. Given the information, it is not explicitly stated that a song on "Fantastic Voyage" samples a Stevie Wonder song, but "Gangsta's Paradise" does. \\
        \textbf{Answer:} I don't know \quad \textbf{(Reward: 0)} \\
        \hline
    \end{tabular}
    \caption{Comparison on an Answerable Case. FaithRL correctly bridges the evidence, whereas GRPO fails to link the compilation album to the original track.}
    \label{tab:case_answerable}
\end{table}

\begin{table}[h]
    \centering
    \small
    \begin{tabular}{p{0.95\textwidth}}
        \hline
        \textbf{Question:} Where did Henri Christophe and other slaves hold an uprising from 1791 to 1804 that led to the founding of a state which was both free from slavery and ruled by non-whites and former captives? \\
        \hline
        \textbf{FaithRL (Ours) Output:} \\
        1. First, identify the state founded by the uprising. From Document 'State of Haiti', we learn "The State of Haiti (French: État d'Haïti, Haitian: Leta an Ayiti) was the name of the state in northern Haiti. It was created on October 17, 1806 following the overthrow of the Empire of Haiti following the assassination of Emperor Jacques I. The northern State of Haiti was ruled by Henri Christophe originally as 'Provisional Chief of the Haitian Government' from October 17, 1806 until February 17, 1807 when he became 'President of the State of Haiti'. The 1807 constitution for the State of Haiti made the post of president a position for life with the president having the power to appoint his successor. On March 28, 1811 President Henri was proclaimed King Henry I, thereby dissolving the State of Haiti and creating the Kingdom of Haiti." This document mentions the State of Haiti but does not directly state it was free from slavery and ruled by non-whites and former captives. However, it does refer to the Haitian Revolution, which is mentioned in other documents. Document 'Jean-Baptiste Sans Souci' states "Jean-Baptiste Sans-Souci was a leader of rebel slaves during the Haitian Revolution. He was assassinated by rival black rebel leader, Henri Christophe, in 1803, shortly before Haiti won its independence." This suggests the Haitian Revolution led to the founding of Haiti, which was free from slavery and ruled by non-whites and former captives. Document 'Henri Christophe' confirms Henri Christophe was a "key leader in the Haitian Revolution, which succeeded in gaining independence from France in 1804." Document 'State of Haiti' also mentions the Haitian Revolution in relation to the founding of the State of Haiti. Therefore, the Haitian Revolution is the uprising referred to in the question.\\
        2. The question asks where Henri Christophe and other slaves held an uprising. From Document 'Jean-Baptiste Sans Souci', we learn "Jean-Baptiste Sans-Souci was a leader of rebel slaves during the Haitian Revolution. He was assassinated by rival black rebel leader, Henri Christophe, in 1803, shortly before Haiti won its independence." This suggests the uprising took place in Haiti. However, the document does not specify a particular location within Haiti. Given the information, it is clear the uprising occurred in Haiti, but the exact location is not specified.\\
        3. The question asks for the location of the uprising, but the documents do not provide a specific location within Haiti. Therefore, I do not have enough information to pinpoint the exact location of the uprising.\\
        \textbf{Answer:} I don't know \quad \textbf{(Reward: $y_0$)} \\
        \hline
        \textbf{GRPO (Baseline) Output:} \\
        1. First, we need to identify the region where Henri Christophe and other slaves held an uprising from 1791 to 1804. From the document "Henri Christophe", we learn that Henri Christophe was a key leader during the Haitian slave rebellion (1791–1804).\\
        2. The Haitian Revolution is the uprising we are looking for. Next, we need to find out if the state that was founded was free from slavery and ruled by non-whites and former captives.\\
        3. The document "Jean-Baptiste Sans Souci" mentions that Jean-Baptiste Sans-Souci was a leader of rebel slaves during the Haitian Revolution, and that Haiti won its independence from France in 1804.\\
        4. The document "State of Haiti" provides more information about the State of Haiti, which was created in northern Haiti after the overthrow of the Empire of Haiti following the assassination of Emperor Jacques I. It was ruled by Henri Christophe.\\
        5. The document "Kingdom of Haiti" confirms that Henri Christophe established the Kingdom of Haiti in 1811, which was in the north of Haiti.\\
        6. Given that Henri Christophe was a former slave and a key leader in the Haitian Revolution, and that the documents mention Haiti gaining independence and becoming a state ruled by non-whites (former slaves), we can conclude that the state in question is Haiti. \\
        \textbf{Answer:} Haiti \quad \textbf{(Reward: $-x_0$)} \\
        \hline
    \end{tabular}
    \caption{Comparison on an Unanswerable Case. FaithRL recognizes the missing specific evidence and refuses, while GRPO hallucinates/guesses a related entity.}
    \label{tab:case_unanswerable}
\end{table}

\section{Prompt Template}
\label{app:prompt_template}

We provide the comprehensive prompt templates utilized throughout our framework: Table~\ref{tab:evidence_prompt} details the transformation of question decompositions into declarative evidence; Table~\ref{tab:reasoning_prompt} defines the instruction for the reasoning policy; 
Table~\ref{tab:faithfulness_prompt} and Table~\ref{tab:ood_outcome_prompt} outline the step-level faithfulness verification protocol for multihop reasoning tasks and OOD math tasks, respectively;
Table~\ref{tab:outcome_prompt} specifies the criteria for judging final answer correctness .

\begin{table*}[!ht]
\caption{Prompt for converting question decomposition into declarative evidence statements.}
\centering
\begin{prompt}[title=Evidence Construction Prompt]
{\bf System Input:} \\
You are an expert at turning subquestion-answer pairs into one clear, natural English statement.

Instructions:
1. Produce ONE grammatically correct declarative sentence using the given subquestion and its answer.
2. Preserve the semantic relationship between entities.
3. If the subquestion contains references like \#1, \#2, etc., resolve them using the list of prior answers (1-indexed).
4. Return ONLY the resulting sentence, nothing else.

Examples:
Q: who wrote crazy little thing called love original artist | A: Freddie Mercury
$\rightarrow$ Freddie Mercury wrote Crazy Little Thing Called Love.

Q: In what year did \#1 die? | A: 1991 | prior: ['Freddie Mercury']
$\rightarrow$ Freddie Mercury died in 1991.

{\bf User Input:} \\
Convert the following pair into a single statement.

Subquestion: \{question\}

Answer: \{answer\}

Prior answers (for \# references): \{prior\_answers\}
\end{prompt}
\label{tab:evidence_prompt}
\end{table*}

\begin{table*}[!ht]
\caption{Prompt used for generating reasoning chains (Reasoning Policy).}
\centering
\begin{prompt}[title=Reasoning Prompt]
{\bf System Prompt:} \\
You are a helpful assistant. You are given a Question and References. The references may or may not help answer the question. Your task is to answer the question based on factual information in the references or your own knowledge. If, after examining and reasoning over the References, you are uncertain or don't know the answer, output I don't know as specified below.

\textbf{CRITICAL - You MUST follow this EXACT format:} \\
$<$think$>$ Your thinking process$<$/think$>$ \\
$<$answer$>$Your final answer$<$/answer$>$

\textbf{Rules (STRICTLY ENFORCED):}
1. Put reasoning in $<$think$><$/think$>$ tags
2. NEVER start with anything other than $<$think$>$ or $<$answer$>$
3. The $<$answer$>$ tag MUST contain your final answer or "I don't know"

{\bf User Prompt:} \\
\textbf{References:} \\
\{documents\_context\}

\textbf{Question:} \\
\{question\}
\end{prompt}
\label{tab:reasoning_prompt}
\end{table*}

\begin{table*}[!ht]
\caption{Prompt for the LLM-based Faithfulness Verifier $\mathcal{V}$ (Step-level).}
\centering
\begin{prompt}[title=FaithfulnessVerification Prompt]
{\bf System Input:} \\
You are a strict reasoning consistency judge. Decide if the reasoning segment is FULLY SUPPORTED by the provided evidences.

Rules:
1) Output only one digit: 1 if the segment contains meaningful reasoning AND is strictly supported by the evidences; 0 otherwise.
2) Give 0 if the segment adds NO new information, is just a plan/re-statement, or lacks specific details (e.g., 'We should review the list...').
3) Give 1 ONLY if the segment's key assertion semantically matches or is directly inferred from an evidence.
4) Base the decision strictly on the evidences; ignore world knowledge.
5) Do not provide explanations.

{\bf User Input:} \\
Evidences: \\
\{evd\_text\}

Reasoning Segment: \\
\{seg\_text\}

Output (only 0 or 1):
\end{prompt}
\label{tab:faithfulness_prompt}
\end{table*}

\begin{table*}[!ht]
\caption{Prompt for the LLM-based Faithfulness Verifier $\mathcal{V}$ (OOD tasks).}
\centering
\begin{prompt}[title=Faithfulness Verifier Prompt]
{\bf Input:} \\

You are a strict reasoning consistency judge. Decide if the reasoning segment is FULLY SUPPORTED by the provided evidences.\\
Rules:\\
1) Output only one digit: 1 if the segment contains meaningful reasoning AND is strictly supported by the evidences; 0 otherwise.\\
2) Give 0 if the segment adds NO new information, is just a plan/re-statement, or lacks specific details (e.g., 'We should review the list...').\\
3) Give 1 ONLY if the segment's key assertion semantically matches or is directly inferred from an evidence.\\
4) Base the decision strictly on the evidences; ignore world knowledge.\\
5) Do not provide explanations.\\
Output ONLY one digit: -1 or 1. No text, no explanation. \\
You should make the judgment based on provided examples.

{\bf User Input:} \\
Evidences: \{evidence\} \\
Reasoning Segment: \{reasoning\_segment\} \\
Output (only 0 or 1):
\end{prompt}
\label{tab:ood_outcome_prompt}
\end{table*}

\begin{table*}[!ht]
\caption{Prompt for the Outcome Correctness Judge.}
\centering
\begin{prompt}[title=Outcome Judge Prompt]
{\bf Input:} \\
Assume you are a human expert in grading predictions given by a model. You are given a question and a model prediction. Judge if the prediction matches the ground truth answer by following these steps:

1: Take it as granted that the Ground Truth is always correct. \\
2: If the Prediction exactly matches the Ground Truth, ``score" is 1. \\
3: If the Prediction does not exactly match the Ground Truth, go through the following steps and likely give a score as -1. \\
4: If the Ground Truth is a number, ``score" is 1 if and only if the Prediction gives a number that almost exactly matches the ground truth. \\
5: If the Prediction is self-contradictory, ``score" must be -1. \\
6: If the prediction is not answering the question, ``score" must be -1. \\
7: If the prediction is a concise and correct summary of the ground truth, ``score" is 1. \\
8: If ground truth contains a set of items, prediction must contain exactly same items for the score to be 1. \\
9: Otherwise, ``score" is -1.

Output ONLY one digit: -1 or 1. No text, no explanation. \\
You should make the judgment based on provided examples.

\textbf{Examples:} \\
Question: When did the director of film Lord Richard In The Pantry die? \\
Ground Truth: 7 January 1984 \\
Prediction: January 7, 1984 \\
Output: 1

Question: Who is older, Charles Badham or Médéric De Vasselot De Régné? \\
Ground Truth: Charles Badham \\
Prediction: Médéric De Vasselot De Régné \\
Output: -1

{\bf User Input:} \\
Question: \{question\} \\
Ground Truth: \{ground\_truth\_answer\} \\
Prediction: \{predicted\_answer\} \\
Output:
\end{prompt}
\label{tab:outcome_prompt}
\end{table*}

\end{document}